\documentclass{article}

\PassOptionsToPackage{numbers}{natbib}
\usepackage[main,final]{neurips_2026}

\usepackage{parskip}
\usepackage{amsmath}
\usepackage{amsthm}
\date{}
\usepackage[utf8]{inputenc} 
\usepackage[T1]{fontenc}    
\usepackage{url}            
\usepackage{booktabs}       
\usepackage{amsfonts}       
\usepackage{amsmath}        
\usepackage{amssymb}        
\usepackage{nicefrac}       
\usepackage{microtype}      
\usepackage{xcolor}         
\usepackage{graphicx}       
\usepackage{tikz}
\usepackage{tikz-3dplot}
\usepackage{subcaption}
\usepackage{hyperref}
\hypersetup{colorlinks=true, linkcolor=blue, citecolor=blue, urlcolor=blue}
\usetikzlibrary{shapes, arrows, positioning, calc, patterns, backgrounds, arrows.meta, automata}
\usepackage{algorithm}
\usepackage{algpseudocode}
\usepackage{tikz-3dplot}
\usepackage{wrapfig}
\usepackage{natbib}
\usepackage{amsmath}
\usepackage{amssymb}
\usepackage{amsthm}
\usepackage{tikz}
\usepackage{cleveref}
\usetikzlibrary{automata, positioning, decorations.pathreplacing, arrows.meta, calc}

\usepackage{pifont}

\newcommand{\abs}[1]{\left| {#1} \right|}

\def\\X{{\mathcal{\X}}}
\def\A{{\mathcal{A}}}

\definecolor{greenp}{rgb}{0.0, 0.51, 0.5}

\newcommand{\initial}{\nu_0}
\newcommand{\bcc}[1]{\left\{{#1}\right\}}

\newcommand{\bs}[1]{\left[{#1}\right]}
\newcommand{\norm}[1]{\left\| {#1} \right\|}

\usepackage{dsfont}
\usepackage{thm-restate}

\usepackage{mathtools} 
\newcommand{{\transpose}}{^\mathsf{\scriptscriptstyle T}}

\usepackage{mathtools}  

\newcommand{\piout}{\pi_{\mathrm{out}}}

\newcommand{\bc}[1]{\left\{{#1}\right\}}

\newcommand{\X}{\mathcal{X}}

\newtheorem{definition}{Definition}[section]
\newtheorem{theorem}{Theorem}[section]
\newtheorem{lemma}{Lemma}[section]

\newtheorem{corollary}{Corollary}[section]

\newcommand{\br}[1]{\left({#1}\right)}  

\def\argmax{\mathop{\mbox{ arg\,max}}}

\renewcommand{\phi}{\varphi}

\newcommand{\pidata}{\pi_{\mathrm{data}}}

\usepackage{pifont}

\usepackage[most]{tcolorbox}
\usepackage{algorithm}
\usepackage{algorithmicx}
\usepackage{algpseudocode}

\usepackage{tocloft}
\addtocontents{toc}{\protect\setcounter{tocdepth}{0}}

\title{\textbf{Split the Differences, Pool the Rest: \\ Provably Efficient Multi-Objective Imitation}}

\author{
  \begin{tabular}{cc}
    Ziyad Sheebaelhamd\thanks{Equal contribution.} & Luca Viano\footnotemark[1]~\thanks{Work done while at EPFL.} \\
    \textmd{University of Tübingen} & \textmd{ETH AI Center} \\
    \texttt{ziyad.sheebaelhamd@uni-tuebingen.de} & \texttt{luca.viano@ai.ethz.ch} \\[1.5em]
    Volkan Cevher & Claire Vernade \\
    \textmd{EPFL} & \textmd{University of Technology Nuremberg} \\
    \texttt{volkan.cevher@epfl.ch} & \texttt{claire.vernade@utn.de}
  \end{tabular}
}

\begin{document}

\maketitle

\begin{abstract}
This work investigates multi-objective imitation learning: the problem of recovering policies that lie on the Pareto front given demonstrations from multiple Pareto-optimal experts in a Multi-Objective Markov Decision Process (MOMDP). Standard imitation approaches are ill-equipped for this regime, as naively aggregating conflicting expert trajectories can result in dominated policies. To address this, we introduce Multi-Output Augmented Behavioral Cloning (MA-BC), an algorithm that systematically partitions divergent expert data while pooling state-action pairs where no behavior conflict is observed. Theoretically, we prove that MA-BC converges to Pareto-optimal policies at a faster statistical rate than any learner that considers each expert dataset independently. Furthermore, we establish a novel lower bound for multi-objective imitation learning, demonstrating that MA-BC is minimax optimal. Finally, we empirically validate our algorithm across diverse discrete environments and, guided by our theoretical insights, extend and evaluate MA-BC on a continuous Linear Quadratic Regulator (LQR) control task.
\end{abstract}

\section{Introduction}
Real-world datasets for complex sequential decision-making tasks, such as those used for LLM fine-tuning \cite{Maarten2020Social,Adina2018broad,cui2023ultrafeedback,zheng2023judging} and robotic learning \cite{o2024open,khazatsky2024droid}, are rarely generated by a single, monolithic expert. Instead, they are typically sourced from diverse human or algorithmic operators, each acting under a distinct, often implicit, utility function over the underlying objectives. This inherent variation in operator preferences naturally yields highly multimodal datasets. The dominant paradigm addresses this by pooling these diverse trajectories into a single dataset and training expressive, multimodal models in an attempt to capture the aggregate variance of the constituent operators \cite{rlhf,o2024open,walke2023bridgedata,diff_policy, vqbet, qfat}. However, the foundational theory of imitation learning in these conflicting, multi-objective settings remains largely underexplored.

In this work, we adopt an explicit treatment of this inherent multi-objectivity. We investigate the problem of imitation learning from multiple diverse experts, where each expert optimizes a distinct utility over the underlying vector rewards of a shared sequential decision task. We formalize this setting as a Multi-Objective Markov Decision Process (MOMDP) \cite{white1982multi,furukawa1980characterization,roijers2013survey} and assume that the expert policies lie on the Pareto front.

A common practice, for example in LLM pretraining \cite{radford2018improving,devlin2019bert} and robotics \cite{vla, vqbet,diff_policy, qfat}, is to naively pool expert data to train a single joint model. This approach presents a challenge: we formally prove that such aggregation can yield dominated policies that fall entirely off the Pareto frontier (\Cref{thm:single_output_BC}). As an example, consider an autonomous driving dataset collected from two experts: one prioritizing speed, the other safety. A jointly trained model will learn a compromised policy that drives fast in some states and cautiously in others. Such a policy ultimately fails to optimize any coherent, monotonically increasing utility function of the returns and results in consistently suboptimal behavior.

An alternative approach, which has appeared in some form in \cite{kim2024pareto,korbak2023pretraining}, is to reduce the problem to independent, single-expert imitation learning and train a separate model for each expert dataset. While this guarantees the recovered policies lie on the Pareto front, it entirely ignores the underlying structure of the MOMDP. Because the experts are constrained to the same Pareto frontier, their optimal policies inherently share substantial state-action overlap.
Treating them as independent tasks discards valuable cross-dataset information, leading to severe sample inefficiency.
Moreover, the only existing approach \cite{franzmeyer2024select} that interpolates between merging all the data and splitting it still suffers from at least one of the two issues above.
This raises the question:
\begin{tcolorbox}[
  colback=gray!10!white,
  colframe=gray!80!black,
  width=\linewidth,
  before skip=10pt,
  after skip=10pt,
  halign=center
]
\emph{Can data from multiple experts be used to efficiently learn a policy on the Pareto front?}
\end{tcolorbox}

We answer this question positively for discrete MOMDPs and introduce Multi-Output Augmented Behavioral Cloning (MA-BC), a Multi-Objective Imitation Learning (MOIL) algorithm designed to guarantee Pareto optimality while exploiting the shared structure on the frontier to achieve fast statistical rates. At its core, MA-BC operates on a simple yet powerful principle: by isolating the state-action pairs where experts disagree and aggregating the data where no conflict was observed, MA-BC efficiently reconstructs the optimal frontier policies that generated the dataset. We further propose a novel lower bound for Multi-Objective Imitation Learning, and show that MA-BC is minimax optimal. Motivated by our theory in the discrete domains, we propose an extension of MA-BC to continuous state and action spaces.

\paragraph{Contributions and outline} In Section \ref{sec:preliminaries}, we formalize the Multi-Objective Imitation Learning (MOIL) problem and contribute a structural result of independent interest: the ``\emph{Pareto Path Guarantee}'', which proves that Pareto-optimal policies remain connected via single-state action changes. Building on this foundation, Section \ref{sec:algo} introduces Multi-Output Augmented Behavioral Cloning (MA-BC), a novel algorithm that systematically partitions divergent expert data to maintain Pareto optimality. We provide theoretical guarantees proving that MA-BC converges to the Pareto front at a faster statistical rate than independent learners, and we establish a novel lower bound demonstrating its minimax optimality. Next, Section \ref{sec:continous_mabc} extends our framework to accommodate continuous state-action spaces. Finally, Section \ref{sec:experiments} empirically validates MA-BC across diverse discrete environments and demonstrates its efficacy on a continuous 6-DOF quadcopter control task. For a broader discussion of the related literature on imitation learning, we refer the reader to \Cref{app:related}.

\section{Preliminaries}
\label{sec:preliminaries}
In this section we establish the background for MOMDPs, focusing on the structural properties of Pareto front policies, and formalize the multi-objective imitation learning problem.

\subsection{Multi-Objective Markov Decision Processes (MOMDPs)}
\label{sec:momdps}

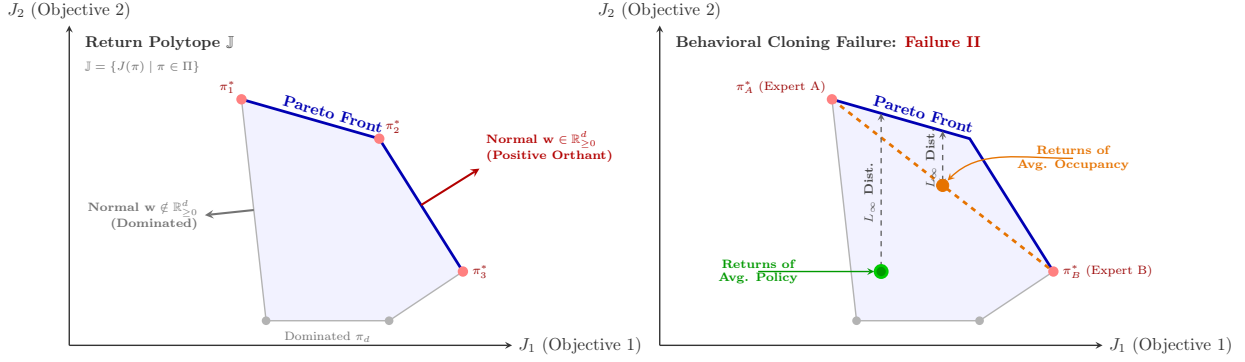
\begin{figure*}[tbp]
    \centering
    \resizebox{\textwidth}{!}{
    \begin{tikzpicture}[>=stealth, scale=1.0]
        \begin{scope}[shift={(0,0)}]
            \draw[->, thick, black!80] (0,0) -- (9.0,0) node[right] {$J_1$ (Objective 1)};
            \draw[->, thick, black!80] (0,0) -- (0,6.5) node[above] {$J_2$ (Objective 2)};
            \coordinate (Pi1) at (3.5, 5.0);
            \coordinate (Pi2) at (6.3, 4.2);
            \coordinate (Pi3) at (8.0, 1.5);
            \coordinate (Pi4) at (6.5, 0.5);
            \coordinate (Pi5) at (4.0, 0.5);
            \fill[blue!5] (Pi1) -- (Pi2) -- (Pi3) -- (Pi4) -- (Pi5) -- cycle;
            \draw[gray!60, thick] (Pi1) -- (Pi5) -- (Pi4) -- (Pi3);
            \draw[blue!70!black, ultra thick] (Pi1) -- (Pi2) -- (Pi3);
            \node[anchor=north west, align=left, font=\bfseries\small, text=black!80] at (0.2, 6.4) {Return Polytope $\mathbb{J}$};
            \node[anchor=north west, align=left, font=\scriptsize, black!60] at (0.2, 5.9) {$\mathbb{J} = \{ J(\pi) \mid \pi \in \Pi \}$};
            \node[blue!70!black, font=\bfseries\small, rotate=-16, anchor=south] at (5.3, 4.48) {Pareto Front};
            \fill[red!50] (Pi1) circle (3pt) node[above left, font=\scriptsize, text=red!60!black] {$\pi_1^*$};
            \fill[red!50] (Pi2) circle (3pt) node[above right, font=\scriptsize, text=red!60!black] {$\pi_2^*$};
            \fill[red!50] (Pi3) circle (3pt) node[right, font=\scriptsize, text=red!60!black, xshift=2pt] {$\pi_3^*$};
            \fill[gray!60] (Pi4) circle (2.5pt);
            \fill[gray!60] (Pi5) circle (2.5pt);
            \node[gray!90!black, font=\scriptsize, anchor=north] at (5.25, 0.4) {Dominated $\pi_d$};
            \coordinate (MidPareto) at ($(Pi2)!0.5!(Pi3)$);
            \draw[->, red!70!black, very thick] (MidPareto) -- ++(32:1.4cm)
                node[pos=1.0, anchor=south west, align=left, font=\scriptsize\bfseries]
                {Normal $\mathbf{w} \in \mathbb{R}^d_{\ge 0}$\\(Positive Orthant)};
            \coordinate (MidDom) at ($(Pi1)!0.5!(Pi5)$);
            \draw[->, gray!90!black, very thick] (MidDom) -- ++(186:1.0cm)
                node[pos=1.0, anchor=east, align=right, font=\scriptsize\bfseries]
                {Normal $\mathbf{w} \notin \mathbb{R}^d_{\ge 0}$\\(Dominated)};
        \end{scope}
        \begin{scope}[shift={(12.0, 0)}]
            \draw[->, thick, black!80] (0,0) -- (9.0,0) node[right] {$J_1$ (Objective 1)};
            \draw[->, thick, black!80] (0,0) -- (0,6.5) node[above] {$J_2$ (Objective 2)};
            \coordinate (PiA) at (3.5, 5.0);
            \coordinate (PiMid) at (6.3, 4.2);
            \coordinate (PiB) at (8.0, 1.5);
            \coordinate (Pi4) at (6.5, 0.5);
            \coordinate (Pi5) at (4.0, 0.5);
            \coordinate (PiAvgOcc) at ($(PiA)!0.5!(PiB)$);
            \coordinate (PiIL) at (4.5, 1.5);
            \fill[blue!5] (PiA) -- (PiMid) -- (PiB) -- (Pi4) -- (Pi5) -- cycle;
            \draw[gray!60, thick] (PiB) -- (Pi4) -- (Pi5) -- (PiA);
            \draw[blue!70!black, ultra thick] (PiA) -- (PiMid) -- (PiB);
            \draw[orange!90!black, ultra thick, dashed] (PiA) -- (PiB);
            \coordinate (ProjAvgOcc) at (5.75, 4.357);
            \draw[<->, thick, black!60, dashed] (PiAvgOcc) -- (ProjAvgOcc)
                node[midway, sloped, above, font=\scriptsize\bfseries, text=black!70] {$L_\infty$ Dist.};
            \coordinate (ProjIL) at (4.5, 4.714);
            \draw[<->, thick, black!60, dashed] (PiIL) -- (ProjIL)
                node[midway, sloped, above, font=\scriptsize\bfseries, text=black!70] {$L_\infty$ Dist.};
            \node[anchor=north west, align=left, font=\bfseries\small, text=black!80] at (0.2, 6.4) {Behavioral Cloning Failure: \textbf{\textcolor{red!70!black}{Failure II}} };
            \node[blue!70!black, font=\bfseries\small, rotate=-16, anchor=south] at (5.3, 4.48) {Pareto Front};
            \filldraw[orange!90!black] (PiAvgOcc) circle (3.5pt);
            \coordinate (LblAvgOcc) at (8.4, 3.8);
            \node[text=orange!90!black, font=\scriptsize\bfseries, align=center] at (LblAvgOcc) {Returns of\\Avg. Occupancy};
            \draw[->, orange!90!black, thick, shorten >= 4pt] (LblAvgOcc) to[out=180, in=45] (PiAvgOcc);
            \filldraw[green!60!black, draw=green!80!black, ultra thick] (PiIL) circle (3.5pt);
            \coordinate (LblAvgPol) at (2.0, 1.5);
            \node[text=green!60!black, font=\scriptsize\bfseries, align=center] at (LblAvgPol) {Returns of\\Avg. Policy};
            \draw[->, green!60!black, thick, shorten >= 4pt] (LblAvgPol) to[out=0, in=180] (PiIL);
            \fill[red!50] (PiA) circle (3pt) node[above left, font=\scriptsize, text=red!60!black] {$\pi_A^*$ (Expert A)};
            \fill[red!50] (PiB) circle (3pt) node[right, font=\scriptsize, text=red!60!black, xshift=2pt] {$\pi_B^*$ (Expert B)};
            \fill[gray!60] (Pi4) circle (2.5pt);
            \fill[gray!60] (Pi5) circle (2.5pt);
        \end{scope}
    \end{tikzpicture}
    }

    \vspace{0.2cm}

    \caption{\textbf{(Left) Geometry of the MOMDP Return Space.} The space of expected returns $\mathbb{J}$ is a convex polytope \cite{momdp_convexity}. The Pareto front (bold blue line) consists of boundary faces whose outward normal vectors $\mathbf{w}$ lie entirely within the positive orthant $\mathbb{R}^d_{\ge 0}$. Boundary faces with outward normal vectors containing negative components represent Pareto-dominated policies. All vertex policies are deterministic. \textbf{(Right) Naive Behavioral Cloning Failure.} Applying standard BC to a mixed dataset from multiple Pareto-optimal experts ($\pi_A^*$, $\pi_B^*$) converges to either the average occupancy measure (orange dot) or the average policy (green dot), depending on whether the objective is optimized at the trajectory or state level, respectively. Due to potentially nonlinear environment dynamics, the expected return of the average policy can ``sag'' significantly into the interior of the polytope. Nonetheless, both outcomes maintain a constant suboptimality gap as visualized by their $L_\infty$ distance to the Pareto front (dashed black lines), highlighting \textbf{\textcolor{red!70!black}{Failure II}}.}
    \label{fig:momdp_geometry_and_failure}
\end{figure*}

To understand why naive imitation learning fails in multi-objective settings and why selective data partitioning succeeds, we must first establish the structural and geometric properties of MOMDPs.

\paragraph{Setup} We define an MOMDP as a tuple $\mathcal{M}=(\X,\mathcal{A},P,\gamma,\initial,r)$, where $\X$ and $\mathcal{A}$ are the discrete state and action spaces, $P$ is the transition dynamics, $\gamma \in [0, 1)$ is the discount factor, and $\initial$ is the initial state distribution \cite{white1982multi,furukawa1980characterization}.
As in the single objective setting, we consider controlling the system with a stationary Markov policy $\pi: \X\rightarrow\Delta_{\A}$ which is a mapping from states to distributions over the action space. We denote the set of all stationary Markov policies as $\Pi$.
Crucially, the reward function $r:\X\times \mathcal{A}\rightarrow [0,1]^{d}$ outputs a $d$-dimensional vector representing potentially conflicting objectives.  The expected vector-valued return of a stationary policy $\pi$, denoted as $J(\pi) \in \mathbb{R}^d$, is given by $J(\pi)=\mathbb{E}_{\pi}\left[\sum_{t=0}^{\infty}\gamma^{t}r(X_t,A_{t})\right]$ where the infinite length sequence $\br{X_t, A_t}^\infty_{t=0}$ is generated by sampling $X_0 \sim \initial$, and then for any $t \geq 0$, $A_t \sim \pi(\cdot|X_t)$ and $X_{t+1} \sim P(\cdot|X_t, A_t)$. Moreover, we define the set of possible expected vector returns as $\mathbb{J} = \bc{ J(\pi) : \pi: \X\rightarrow \Delta_{\A}}$. An important concept for our analysis is the occupancy measure induced by policy $\pi$ defined as $\mu^\pi(x,a) = (1-\gamma) \mathbb{E}_{\pi}\bs{\sum_{t=0}^{\infty}\gamma^{t}\mathds{1}_{\bc{X_t = x,A_{t}=a}}}$ and the state occupancy measure defined as the marginal $\nu^\pi(x) = \sum_{a\in\A} \mu^\pi(x,a)$.
\paragraph{Solution concept}
In general MOMDPs, there rarely exists a single policy $\pi^\star$ such that $\forall i \in [d]$, $J_i(\pi^\star) = \max_{\pi\in\Pi} J_i(\pi)$. For this reason, we focus on learning \emph{Pareto front policies} which are guaranteed to exist in any MOMDP. In particular, we define an approximate Pareto front policy as follows.
\begin{definition}\textbf{$\varepsilon$-approximate Pareto front policy ($\varepsilon$-PF-policy)} Let $\pi^\star$ be the output of a randomized algorithm. Then, $\pi^\star$ is an $\varepsilon$-PF-policy if$$ \nexists \pi \in \Pi ~~ \text{such that} ~~ (\forall i \in [d], ~~ J_i(\pi) - \mathbb{E}[J_i(\pi^\star)] \ge \varepsilon) ~~ \text{and} ~~ (\exists j \in [d], ~~ J_j(\pi) - \mathbb{E}[J_j(\pi^\star)] > \varepsilon)\,, $$\end{definition}

where the expectation is with respect to the algorithm's randomization. In words, a policy is an $\varepsilon$-PF-policy if there does not exist another policy in $\Pi$ that improves all entries of $\mathbb{E}[J(\pi^\star)]$ by at least $\varepsilon$, and improves at least one entry by strictly more than $\varepsilon$. Moreover, if $\pi^\star$ is not random and is an $\varepsilon$-PF-policy with $\varepsilon=0$, we call it a Pareto front policy, abbreviated as PF-Policy. In any MOMDP, the set of PF-policies is guaranteed to be non-empty \cite{furukawa1980characterization,white1982multi}.
\paragraph{Structural properties} In recent work \cite{momdp_convexity}, it has been shown that the return space achieved by the returns of stochastic stationary policies forms a convex polytope, where the vertices are deterministic policies. This implies that the Pareto front lies on the boundary of the polytope, and corresponds to edges and vertices that have a normal vector in the positive orthant. This can be visualized in Figure~\ref{fig:momdp_geometry_and_failure}. Moreover, we define neighboring PF-policies as follows.
\begin{definition}[Neighboring PF-policies]
\label{def:neighboring_vertices}
    Two PF-policies $\pi_A$ and $\pi_B$ are said to be \textbf{neighboring} if and only if there exists a weight vector $w \in \mathbb{R}^d_{+}$ such that the face of optimal returns, i.e. $\argmax_{J \in \mathbb{J}} w^T J$ coincides exactly with the line segment connecting $J(\pi_A)$ and $J(\pi_B)$.
\end{definition}

Next, we introduce a key structural result that motivates the design of our algorithm in \Cref{sec:algo}: We show that any pair of neighboring PF-policies differing in $k$ states can be bridged by a sequence of $k$ intermediate PF-policies, each differing from the previous one in the action at a single state.

\begin{restatable}{theorem}{Pareto}[Pareto Path Existence]
    \label{thm:pareto_path}
    Let $\pi_A$ and $\pi_B$ be neighboring PF-policies that differ in $k$ states. Then, there exists a finite sequence of $k+1$ deterministic PF-policies $(\pi_0, \pi_1, \dots, \pi_k)$ such that $\pi_0 = \pi_A$ and $\pi_k = \pi_B$ and consecutive policies in the sequence differ in exactly one state, i.e.  $\sum_{x\in \X} \mathds{1}{\bc{\pi_{i-1}(x) \neq \pi_{i}(x)}} = 1$ for all $i \in \{1, \dots, k\}$.
\end{restatable}

We refer to the sequence $(\pi_0, \pi_1, \dots, \pi_k)$ as a \emph{Pareto path}. The existence of this path guarantees that the Pareto front is inherently populated by structurally similar policies, precluding scenarios where all optimal policies have completely disjoint behaviors. To the best of our knowledge, \Cref{thm:pareto_path} is a new geometric characterization of MOMDPs of independent interest. Notably, it generalizes the findings of \cite{finding_pareto_front} by removing their requirement of a unique mapping between policies and returns. The proof is provided in \Cref{app:proofs}, where we contrast our result with \cite{finding_pareto_front} and additionally show that all stochastic mixtures of the policies along the segment $J(\pi_{i-1})$ to $J(\pi_{i})$ are also Pareto optimal. This can be visualized in \Cref{fig:pareto_path}.

\begin{figure}[tbp]
    \centering
    \resizebox{0.7\textwidth}{!}{
    \begin{tikzpicture}[>=stealth, scale=1.0]
        \fill[teal!5, rounded corners] (0, 0.5) rectangle (3.5, 6.7);
        \node[font=\bfseries\small, text=teal!80!black, anchor=north] at (1.75, 6.5) {Policy Space $\Pi$};
        \coordinate (P0) at (1.75, 5.4);
        \coordinate (P1) at (1.75, 4.4);
        \coordinate (P2) at (1.75, 3.4);
        \coordinate (P3) at (1.75, 2.4);
        \coordinate (P4) at (1.75, 1.4);
        \node[circle, draw=teal!80!black, fill=white, thick, inner sep=2pt, minimum size=0.8cm] (N0) at (P0) {$\pi_0$};
        \node[circle, draw=teal!80!black, fill=white, thick, inner sep=2pt, minimum size=0.8cm] (N1) at (P1) {$\pi_1$};
        \node[circle, draw=teal!80!black, fill=white, thick, inner sep=2pt, minimum size=0.8cm] (N2) at (P2) {$\pi_2$};
        \node[circle, draw=teal!80!black, fill=white, thick, inner sep=2pt, minimum size=0.8cm] (N3) at (P3) {$\pi_3$};
        \node[circle, draw=teal!80!black, fill=white, thick, inner sep=2pt, minimum size=0.8cm] (N4) at (P4) {$\pi_4$};
        \node[font=\scriptsize, text=teal!80!black, anchor=east] at (N0.west) {$(=\pi_A)$};
        \node[font=\scriptsize, text=teal!80!black, anchor=east] at (N4.west) {$(=\pi_B)$};
        \draw[->, teal!80!black, thick] (N0) -- (N1) node[midway, right, font=\scriptsize] {1-flip};
        \draw[->, teal!80!black, thick] (N1) -- (N2) node[midway, right, font=\scriptsize] {1-flip};
        \draw[->, teal!80!black, thick] (N2) -- (N3) node[midway, right, font=\scriptsize] {1-flip};
        \draw[->, teal!80!black, thick] (N3) -- (N4) node[midway, right, font=\scriptsize] {1-flip};
        \begin{scope}[shift={(2.0, 0)}]
            \node[font=\bfseries\small, text=black!80, anchor=north west] at (6.0, 6.5) {Return Space $\mathbb{J}$};
            \fill[blue!5] (7.0, 5.4) -- (11.0, 2.0) -- (11.5, 0.5) -- (6.0, 0.5) -- (5.5, 3.5) -- cycle;
            \draw[blue!80!black, ultra thick] (7.0, 5.4) -- (11.0, 2.0);
            \node[font=\bfseries, text=blue!80!black, rotate=-40.4, anchor=south] at (9.0, 3.7) {Pareto Front Edge};
            \coordinate (JA) at (7.0, 5.4);
            \coordinate (JB) at (11.0, 2.0);
        \end{scope}
        \draw[->, thick, black!50, dashed, out=0, in=180] (N0.east) to (JA);
        \draw[->, thick, black!50, dashed, out=0, in=205] (N1.east) to (JA);
        \draw[->, thick, black!50, dashed, out=0, in=225] (N2.east) to (JA);
        \draw[->, thick, black!50, dashed, out=0, in=160] (N3.east) to (JB);
        \draw[->, thick, black!50, dashed, out=0, in=180] (N4.east) to (JB);
        \node[font=\scriptsize\bfseries, text=black!60, fill=white, inner sep=2pt, rounded corners] at (5.9, 3.1) {$J(\cdot)$ Mapping};
        \fill[red] (JA) circle (4.0pt) node[above right, font=\small, text=black] {$J(\pi_A)$};
        \fill[red] (JB) circle (4.0pt) node[right, font=\small, text=black, xshift=4pt] {$J(\pi_B)$};
    \end{tikzpicture}
    }
    \vspace{0.2cm}
    \caption{\textbf{Visualization of Theorem \ref{thm:pareto_path} (Pareto Path Existence).} Let two neighboring Pareto-optimal policies, $\pi_A$ and $\pi_B$, differ from each other at exactly $k=4$ states. As established by the theorem, we can find a sequence of intermediate policies that are a 1-state action flip from each other and are also guaranteed to be on the Pareto front.}
    \label{fig:pareto_path}
\end{figure}
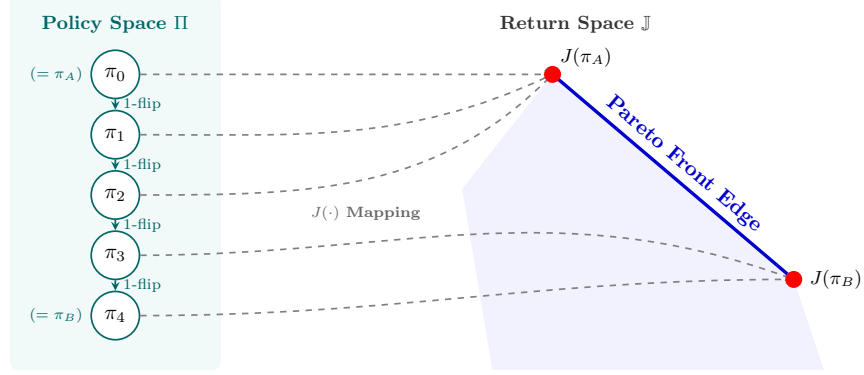

We conclude that the Pareto front is densely populated by structurally similar policies. We will leverage this structural fact in developing our algorithm for imitation learning in MOMDPs that we introduce next.

\subsection{Imitation Learning in MOMDPs}
In imitation learning in MOMDPs, a learning algorithm aims to learn a set of $\varepsilon$-PF policies in an MOMDP with unknown transitions $P$, reward $r$ and initial distribution $\initial$.
The learning signal is a collection of datasets $\bc{\mathcal{D}_\ell }^L_{\ell=1}$. Each dataset contains state-action pairs sampled from the occupancy measure of an expert policy $\pi_\ell\in \Pi$, i.e. $\mathcal{D}_\ell = \bc{X^\ell_n, A^\ell_n}^{N_\ell}_{n=1}$ and $X^\ell_n, A^\ell_n \sim \mu^{\pi_\ell}$ for each $n\in[N_\ell]$.
Then, given that for each $\ell\in[L]$, $\pi_\ell$ is a PF-policy, we aim to find a learning algorithm that outputs a set of policies $\bc{\hat{\pi}_\ell}^{L'}_{\ell=1}$
that are $\varepsilon$-PF.\footnote{Notice that we do not restrict the number of output policies to be equal to the number of experts. That is, we can have $L'\neq L$.}

\paragraph{Assumption (Deterministic Experts)} For ease of exposition, we will assume in the rest of the main text that the expert policies $\pi_\ell$ are deterministic. However, our framework naturally accommodates stochastic experts. We detail the extension to stochastic Pareto policies in Appendix \ref{app:stochastic_experts}.

\paragraph{Reduction to single-expert imitation} The problem can be reduced to single-expert imitation learning: it is indeed possible to learn each $\hat{\pi}_\ell$ by running an imitation learning algorithm for single-objective MDPs, such as behavioral cloning \cite{Pomerleau:1991,foster2024behavior,rajaraman2020toward}. This approach completely neglects the information in the datasets $\bc{\mathcal{D}_j}^L_{j\neq\ell}$, optimizing each policy independently as
$
\hat{\pi}_\ell = \mathrm{argmax}_{\pi\in\Pi} \sum_{X,A \in \mathcal{D}_\ell} \log \pi(A | X).
$

Invoking standard guarantees, such as \cite[Theorem 1.2]{rajaraman2020toward}, we can conclude that $\hat{\pi}_\ell$ is an $\varepsilon$-PF policy if $N_\ell \geq \frac{\abs{\X}}{(1-\gamma)^2\varepsilon}$.  Moreover,  the lower bounds for single-objective MDPs \cite{rajaraman2020toward} imply that no better rate can be achieved by focusing only on $\mathcal{D}_\ell$ as input. Therefore, we have the following negative takeaway.
\begin{tcolorbox}[
  colback=red!10!white,
  colframe=red!70!black,
  width=\linewidth,
  before skip=10pt,
  after skip=10pt
]
\textbf{\textcolor{red!70!black}{Failure I:}} Treating imitation learning in MOMDPs as imitation learning in MDPs requires each expert dataset to have size linear in the state space cardinality, i.e. $\forall \ell \in [L]$, $N_\ell = \Omega(\abs{\X})$.
\end{tcolorbox}

Given that PF-policies agree on a large number of states (see the structural properties highlighted in \Cref{sec:momdps}) there is hope to require that only the size of the joint dataset $\sum^L_{\ell=1} N_\ell$ grows with $\abs{\X}$ as opposed to the much stricter requirement of imposing that each single expert dataset does so.
\paragraph{Single output BC} A tempting approach in this direction is not splitting the data at all, that is, to run behavioral cloning over the full dataset. This generates a policy $\piout = \mathrm{argmax}_{\pi\in\Pi} \sum_{X,A \in \textcolor{teal}{\cup^L_{\ell=1} \mathcal{D}_\ell}} \log \pi(A | X).$ However, this approach suffers from an even more severe failure mode. Indeed, the next theorem proves that single output BC can fail to learn an $\varepsilon$-PF policy even under infinite data.
\begin{restatable}{theorem}{singleBC}
\label{thm:single_output_BC}
    Let $\piout = \mathrm{argmax}_{\pi\in\Pi} \sum_{X,A \in \cup^L_{\ell=1} \mathcal{D}_\ell} \log \pi(A | X)$. There exists an MOMDP and a policy $\pi^\star$ such that:
    $\forall i \in [d]$ it holds that \[
    J_i(\pi^\star) \geq J_i(\piout) + \frac{1}{3(1-\gamma)}\,.
    \]
\end{restatable}
The main idea to prove the negative result is to realize that given $\pi, \pi'$ that are PF-policies, it is not true that the average policy, defined as $\frac{\pi(a|x) + \pi'(a|x)}{2}$ for all $x,a\in \X\times\A$, is a PF-policy. This fact is also depicted in \Cref{fig:momdp_geometry_and_failure}. We can therefore conclude the section with a second negative takeaway.
\begin{tcolorbox}[
  colback=red!10!white,
  colframe=red!70!black,
  width=\linewidth,
  before skip=10pt,
  after skip=10pt
]
\textbf{\textcolor{red!70!black}{Failure II:}} Running imitation learning without splitting the dataset can fail to learn an $\varepsilon$-PF policy even with an infinite amount of data.
\end{tcolorbox}

We now move to the next section to show that splitting the data \emph{in the right manner} avoids both failure cases.

\section{Multi-Output Augmented Behavioral Cloning (MA-BC)}
\label{sec:algo}
Our new algorithm, Multi-Output Augmented Behavioral Cloning (MA-BC), is inspired by the structural result in \Cref{thm:pareto_path} which establishes that PF-policies used by the experts are likely to agree on the same actions in several states. Our idea is then to use this fact to enable a large data reuse mechanism, while avoiding \textcolor{red!70!black}{\textbf{Failure II}}. MA-BC outputs $L' = L$ policies trained by Behavioral Cloning on a curated dataset. In particular, our algorithm first identifies the set of states in which the experts showed contradictory behavior (see Line 2 in \Cref{alg:main}). More precisely, a state is included in $\mathcal{D}_{\mathrm{div}}$ if it appeared in at least two datasets of the collection $\bc{\mathcal{D}_{\ell}}^L_{\ell=1}$ but, in at least one case, the actions taken by two experts in that state are different. Then, in (\Cref{alg:main} Line 3) we define the set of common states as the set difference between the joint dataset $\cup^L_{\ell=1} \mathcal{D}_\ell$
and the diverging dataset $\mathcal{D}_{\mathrm{div}}$. For this operation, we use the convention that removing a state from $\cup^L_{\ell=1} \mathcal{D}_\ell$ means removing all its occurrences together with all the actions taken at that state.

Finally, we learn $L$ policies by running $L$ instances of BC. In particular, to learn the output policy $\hat{\pi}_\ell$, we use the dataset $\mathcal{D}_{\mathrm{common}} \cup \mathcal{D}_\ell$. That is, we use the union of: (i) the expert-specific dataset $\mathcal{D}_\ell$ and (ii) the common dataset $\mathcal{D}_{\mathrm{common}}$.

\paragraph{The rationale behind the design of MA-BC} Crucially,  $\mathcal{D}_{\mathrm{common}}$ also contains information collected by the other experts. In particular, the curated dataset includes states that have not been visited by the expert $\pi_\ell$ but by the other experts. This allows the \emph{experts to help each other} and, in turn, to bypass \textcolor{red!70!black}{\textbf{Failure~I}}. At the same time, \textcolor{red!70!black}{\textbf{Failure~II}} is avoided by including only states in which no expert disagreement has been observed.
\begin{algorithm}[!t]
    \caption{MA-BC: Multi-Output Augmented Behavioral Cloning \label{alg:main}}
    \begin{algorithmic}[1]
    \State \textbf{Input:} Single expert datasets $\bc{\mathcal{D}_\ell}^L_{\ell=1}$.
    \State Identify the diverging states
    \[
    \mathcal{D}_{\mathrm{div}} = \bcc{
        \begin{aligned}
            x \in \X, a \in \A  ~~|~~ & \exists \ell,\ell'\in [L], \ell \neq \ell', \exists ~~n \in [N_\ell], n'\in [N_{\ell'}] \\
            & \text{s.t.} ~~ X^\ell_n=X^{\ell'}_{n'} = x ~~\text{AND}~~ A^\ell_n \neq A^{\ell'}_{n'}
        \end{aligned}
    }
    \]
    \State Define the common states
    \[
    \mathcal{D}_{\mathrm{common}} = (\cup^L_{\ell=1} \mathcal{D}_\ell) \setminus \mathcal{D}_{\mathrm{div}}
    \]
    \State \textcolor{teal}{\% Run BC on the curated dataset}
    \For{$\ell=1,\dots,L$}
    \[
    \hat{\pi}_\ell = \argmax_{\pi\in\Pi} \sum_{X,A \in \mathcal{D}_{\mathrm{common}} \cup \mathcal{D}_\ell } \log \pi(A|X).
    \]
    \EndFor
    \State \textbf{Output:} $\bc{\hat{\pi}_\ell}^L_{\ell=1}$.
    \end{algorithmic}
\end{algorithm}
The avoidance of both failures is made clear by the following theorem which is our main result.

\begin{restatable}{theorem}{MultiBC}
\label{thm:multi_BC}
Let $\bc{\pi_\ell}^L_{\ell=1}$ be Pareto front policies and let $\X_{\mathrm{common}} \subset \X$ be defined as $\X_{\mathrm{common}} := \bc{x \in \X: \pi_1(x) = \pi_2(x) = \dots = \pi_L(x)}$ and let $\bc{\hat{\pi}_\ell}^L_{\ell=1}$ be the output of MA-BC. Then, we have that for all $\ell\in[L]$, $\hat{\pi}_\ell$ is $\varepsilon_\ell$-PF with

\[
\varepsilon_\ell
=
C_{\pi_\ell}
\frac{\abs{\X_{\mathrm{common}}}}
     {e(1-\gamma)^2N}
+
\frac{\abs{\X\setminus\X_{\mathrm{common}}}}
     {e(1-\gamma)^2N_\ell},
\]
where
\[
C_{\pi_\ell}
:=
\norm{
\frac{\nu^{\pi_\ell}}
     {\nu^{\pi_{\mathrm{data}}}}
}_{\infty,\X_{\mathrm{common}}}
:= \max_{x\in\X_{\mathrm{common}}:\,\nu^{\pi_{\mathrm{data}}}(x)>0} \frac{\nu^{\pi_\ell}(x)}{\nu^{\pi_{\mathrm{data}}}(x)},
\]
and $\nu^{\pi_{\mathrm{data}}}$ is defined as
\[
\nu^{\pi_{\mathrm{data}}}(x)
=
\frac{\sum_{j=1}^L N_j\nu^{\pi_j}(x)}{N},
\qquad
\text{for each }x\in\X.
\]
\end{restatable}
We highlight that the concentrability parameter $C_{\pi_\ell}$ is restricted to the common state space $\X_{\mathrm{common}}$. An immediate corollary reads as follows:

\begin{corollary}
\label{cor:MA-BC}
    Fix expert sampling proportions
    $\alpha_1,\ldots,\alpha_L>0$ satisfying
    $\sum_{\ell=1}^L \alpha_\ell=1$, and consider sample sizes
    $N_\ell=\alpha_\ell N$, so that the allocation remains fixed as
    $N$ increases. For any $\varepsilon>0$, let
    $\abs{\X\setminus\X_{\mathrm{common}}}=K$ and assume that
    $\abs{\X}\geq 2K$. Let
    \[
    C_{\pidata}
    :=
    \max_{\ell\in[L]}
    \norm{
        \frac{\nu^{\pi_\ell}}
             {\nu^{\pi_{\mathrm{data}}}}
    }_{\infty,\X_{\mathrm{common}}},
    \qquad
    \nu^{\pi_{\mathrm{data}}}
    =
    \sum_{\ell=1}^L \alpha_\ell\nu^{\pi_\ell}.
    \]
    Then, the whole set $\bc{\hat{\pi}_\ell}$ output by MA-BC is
    $\varepsilon$-PF if
    \[
    N
    \geq
        \frac{2\,C_{\pidata}\abs{\X}}
             {e(1-\gamma)^2\varepsilon}
    \qquad\text{and}\qquad
    N_\ell
    \geq
        \frac{2K}
             {e(1-\gamma)^2\varepsilon}
    \quad\text{for all }\ell\in[L].
    \]
\end{corollary}

The first important remark is that the single dataset sizes $\bc{N_\ell}^L_{\ell=1}$ are now required to scale linearly only with the number of diverging states $K$, which can be much smaller than $\abs{\X}$. In this regime, our approach is expected to be much more efficient than the approach based on imitation learning in MDPs, which suffers from \textcolor{red!70!black}{\textbf{Failure~I}}. Moreover, \textcolor{red!70!black}{\textbf{Failure~II}} is avoided because
\Cref{thm:multi_BC} certifies that the output policies are PF-policies with infinite data, i.e. when  $N \rightarrow \infty$ and $N_\ell \rightarrow \infty$.

\paragraph{On the role of the concentrability coefficient $C_{\pi_\ell}$}
Interestingly, $C_{\pi_\ell}$ in \Cref{thm:multi_BC} determines how much the experts' common dataset is informative for each individual learner. To illustrate this, consider a simple setting with $L=2$.  In this case, we have that $
C_{\pi_1} = \max_{x\in\X_{\mathrm{common}}}\frac{N \nu^{\pi_1}(x)}{N_1 \nu^{\pi_1}(x) + N_2 \nu^{\pi_2}(x)}.$
In the first extreme scenario, assume the two experts are not informative for each other. In particular, let us consider the case in which there exists a state $x\in \X_{\mathrm{common}}$ in which $\nu^{\pi_1}(x) > 0$ but $\nu^{\pi_2}(x) = 0$. Unfortunately, we have that $\norm{
\frac{\nu^{\pi_1}}{\nu^{\pi_{\mathrm{data}}}}
}_{\infty,\X_{\mathrm{common}}} = \frac{N}{N_1}$ and plugging this into our general bound of \Cref{thm:multi_BC} we have that the right hand side reduces to $\frac{\abs{\X}}{e (1-\gamma)^2 N_1}$, which is the same rate we could get using only the data from expert $\pi_1$ (i.e. we fall back to \textbf{\textcolor{red!70!black}{Failure~I}} ). In this situation, the data from the other expert do not help because $C_{\pi_1}$ is as high as it can be.
A more benign situation occurs when, for all $x\in \X_{\mathrm{common}}$ we have that
$
\nu^{\pi_2}(x) \geq c_1 \nu^{\pi_1}(x)$  and $\nu^{\pi_1}(x) \geq c_2 \nu^{\pi_2}(x).$ Intuitively, this means that the two experts induce a similar state occupancy measure even though they choose different actions in some states.
Then, we have that $C_{\pi_1} \leq \frac{N}{N_1 + c_1 N_2} \leq \frac{N}{c_1 N_1 + c_1 N_2} = \frac{1}{c_1}$, where we assumed that $c_1 \leq 1$.
In this case, the right hand side in \Cref{thm:multi_BC} becomes
$\frac{\abs{\X_{\mathrm{common}}} }{e c_1 (1-\gamma)^2 N} + \frac{\abs{\X\setminus\X_{\mathrm{common}}}}{e (1-\gamma)^2 N_1}$,
which is faster than what we would get using the dataset from a unique expert, especially when there are many common states. In this situation we ensure that the second term that scales as $1/N_1$ becomes unaffected by the number of states. At the same time, while the first term depends on the number of states, it exhibits a faster $\mathcal{O}(1/N)$ decay, scaling inversely with the size of the joint dataset. The inverse dependence on $c_1$ implies that the more similar the experts' state occupancy measures are, the faster the first term decays.

\paragraph{Extension to unsplit experts}
Our algorithm assumed access to the split datasets, where the expert identities are revealed. In Appendix \ref{app:unsplit_mabc}, we consider the setting where the dataset is not split, and the learner must first identify the latent experts' identities. We prove a similar bound to Theorem~\ref{thm:multi_BC} in this setting (see Theorem~\ref{thm:unsplit_mabc}).

\paragraph{Is $C_{\pi_\ell}$ necessary?} From the above discussion regarding the role of $C_{\pi_\ell}$, it emerged that MA-BC is not always better than reducing the problem to imitation learning in MDPs. Indeed, when $C_{\pi_\ell}$ takes its largest possible value, each individual dataset $\mathcal{D}_\ell$ is required to grow linearly with $\abs{\X}$, at least according to our analysis in \Cref{thm:multi_BC}.
Therefore, our analysis states that MA-BC is beneficial only when $C_{\pi_\ell}$ is reasonably small.
As we show in the next theorem, the dependence on $C_{\pi_\ell}$ is not an artifact of our analysis, but is rather unavoidable. That is, for any offline imitation learning algorithm, there exists a two-objective MDP in which learning an $\varepsilon$-PF policy requires at least $\Omega(K/\varepsilon)$ state-action pairs in both expert-specific datasets and  $\Omega(C_{\pidata} \abs{\X}/\varepsilon)$ state-action pairs in the joint dataset.

\begin{restatable}{theorem}{L}

\label{thm:lower}
Fix $\gamma\in[1/2,1)$, $N_1,N_2\geq 1$,
$p_1,p_2\in(0,1/2]$, and an index $\ell\in\{1,2\}$.
For any offline expert-indexed multi-output imitation algorithm that receives
\[
\mathcal{D}_j
=
\bc{(X_n^j,A_n^j)}_{n=1}^{N_j},
\qquad
(X_n^j,A_n^j)\sim\mu^{\pi_j},
\]
and outputs $\bc{\hat{\pi}_1,\hat{\pi}_2}$, there exists a
two-objective MOMDP and two deterministic PF-policies
$\pi_1,\pi_2$ such that
\[
\varepsilon_\ell
\geq
\Omega\br{
\min\bcc{
\frac{p_\ell}{1-\gamma},
\;
C_{\pi_\ell}
\frac{\abs{\X_{\mathrm{common}}}}
     {N(1-\gamma)^2}
}
+
\min\bcc{
\frac{1-p_\ell}{1-\gamma},
\;
\frac{\abs{\X\setminus\X_{\mathrm{common}}}}
     {N_\ell(1-\gamma)^2}
}
},
\]
where $\varepsilon_\ell$ is any value for which
$\hat{\pi}_\ell$ is an $\varepsilon_\ell$-PF-policy and
\[
C_{\pi_\ell}
:=
\norm{
\frac{\nu^{\pi_\ell}}
     {\nu^{\pi_{\mathrm{data}}}}
}_{\infty,\X_{\mathrm{common}}}.
\]
Moreover, for the constructed MOMDP,
\[
C_{\pi_\ell}
=
\frac{Np_\ell}{N_1p_1+N_2p_2}.
\]
The universal constant hidden in $\Omega(\cdot)$ is independent of
$N_1,N_2,\gamma,p_1,p_2$, and the state-space cardinalities.
\end{restatable}

Here, \emph{expert-indexed} means that the $\ell$-th output policy is tied to the $\ell$-th demonstrated expert (in the construction of \Cref{app:proof_lower_bound}, this is enforced by fixing the action at the root state). Without such a requirement, an algorithm could output the same PF-policy $L$ times, and no bound involving $N_\ell$ could hold. Consequently, the multiplicative penalty of the concentrability coefficient $C_{\pi_\ell}$ is unavoidable. Specifically, it scales the $\mathcal{O}(1/N)$ term for the expert $\ell^\star = \argmax_{\ell\in[L]} C_{\pi_\ell}$. Note that we exclude single-output imitation learning algorithms from the lower bound because for those, we have already proven a much stronger negative result in \Cref{thm:single_output_BC}, which prevents convergence to PF-policies even under infinite data. In contrast, \Cref{thm:lower} states that the PF is approachable, but at a statistical rate that is no faster than the maximum of two terms, which matches the two terms in the rate obtained by our MA-BC (see \Cref{thm:multi_BC}).

Having developed a complete theoretical understanding of MA-BC in the tabular case, we present its extension to continuous state and action spaces in the next section.

\section{Scaling MA-BC to Continuous State and Action Spaces}
\label{sec:continous_mabc}

Our approach naturally generalizes to continuous state-action spaces by augmenting the dataset of each expert $\mathcal{D}_\ell$ with the following \emph{compatible state action} pairs:

\[
\mathcal{D}^+_\ell = \bc{X',A' \in \cup^L_{\ell' \neq \ell} \mathcal{D}_{\ell'} ~:~\exists X,A \in \mathcal{D}_\ell ~\text{s.t.}~ \norm{X - X'} \leq \delta ~\text{AND}~\norm{A - A'} \leq L^{\pi_{\ell}} \norm{X - X'} }
\]
where $L^{\pi_{\ell}}$ is the Lipschitz constant of the policy of the $\ell^{\mathrm{th}}$ expert, i.e. for all $x,x'\in \X$ it holds that $\norm{\pi_\ell(x) - \pi_\ell(x')} \leq L^{\pi_\ell} \norm{x - x'} $. The quantity $L^{\pi_{\ell}} \norm{X - X'}$ is the minimum meaningful action threshold. Reducing it further risks discarding bias-free samples because if the experts were in fact all equal, we could still expect actions at distance $L^{\pi_{\ell}} \delta$ in states at distance $\delta$. Then, MA-BC outputs the set of policies $\bc{\hat{\pi}_\ell}^L_{\ell=1}$ such that for each $\ell\in [L]$, $\hat{\pi}_\ell \in \mathrm{argmin}_{\pi: \X \rightarrow \A} \sum_{X, A \in \mathcal{D}_\ell \cup \mathcal{D}^+_\ell} \norm{A - \pi(X)}^2$.

\paragraph{A bias-data tradeoff} The parameter $\delta$ tunes a tradeoff between the bias introduced in the regression targets by augmenting the dataset with $\mathcal{D}^+_\ell$ and the size of $\mathcal{D}^+_\ell$ itself. In the extreme case where $\delta = 0$, no bias is introduced but we add some state action pairs in $\mathcal{D}^+_\ell$ only if the same state is visited by two different experts. However, this never happens in a continuous domain. On the contrary, as $\delta$ increases, we allow more and more bias but the size of $\mathcal{D}^+_\ell$ is also guaranteed not to decrease. Indeed, it is easy to notice that every state-action pair that is in the compatible set at level $\delta$ is also compatible for any $\delta' \geq \delta$. We highlight that the extension of MA-BC to continuous states and actions remains a heuristic, as we lack formal guarantees in this setting. However,
we empirically investigate the tradeoffs in this setting in Section \ref{sec:experiments}, highlighting many interesting phenomena.

\section{Experiments}
\label{sec:experiments}
We design a suite of experiments across four distinct domains. First, we evaluate our algorithm on three discrete MOMDP benchmarks: Deep Sea Treasure \cite{deep_sea_env}, Resource Gathering \cite{resource_gathering_env} and Slippery Y-Maze. These environments can be visualized in Figure \ref{fig:discrete_environments}. We further consider a continuous LQR drone navigation task \cite{drone_env}. All environments are described in detail in Appendix \ref{app:environments}. Across all domains, we consistently demonstrate that MA-BC converges to the Pareto front while achieving superior sample efficiency. We further provide a highly visual, pedagogical walkthrough of the Deep Sea Treasure environment in \Cref{app:pedagogical_deep_dive}.
The implementation of all experiments is available at \url{https://github.com/ziyadsheeba/mabc.git}.

\subsection{Sample Efficiency}
\begin{figure*}[tbp]
    \centering
    \begin{subfigure}[b]{0.32\textwidth}
        \centering
        \includegraphics[width=\linewidth]{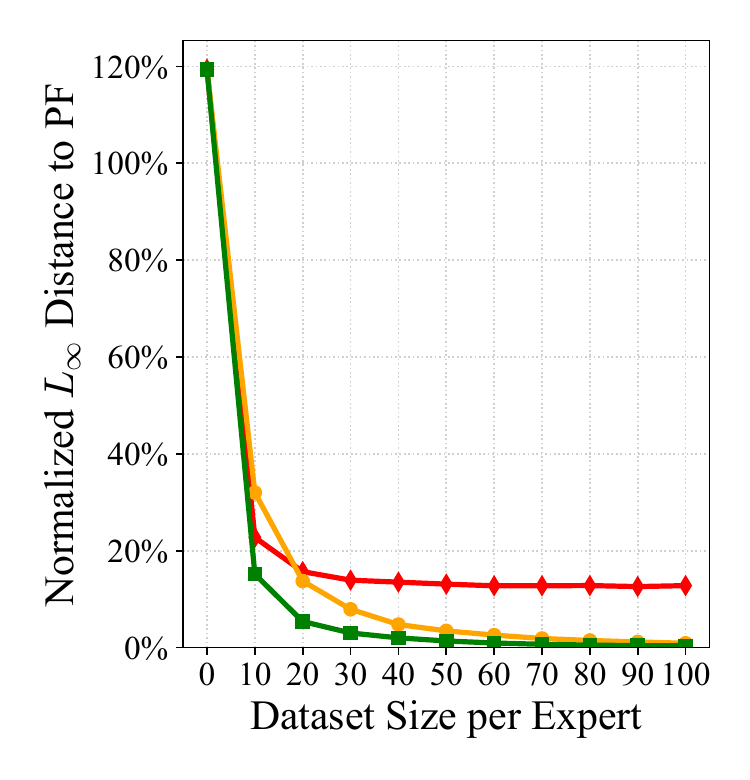}
        \caption{Deep Sea Treasure}
        \label{fig:lc_dst}
    \end{subfigure}
    \hfill
    \begin{subfigure}[b]{0.32\textwidth}
        \centering
        \includegraphics[width=\linewidth]{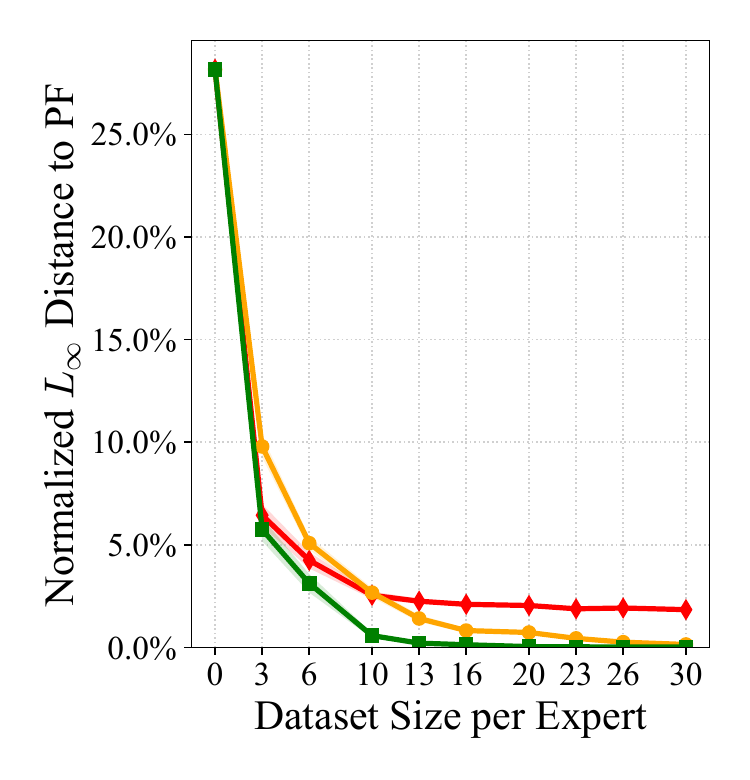}
        \caption{Slippery Y-Maze}
        \label{fig:lc_ymaze}
    \end{subfigure}
    \hfill
    \begin{subfigure}[b]{0.32\textwidth}
        \centering
        \includegraphics[width=\linewidth]{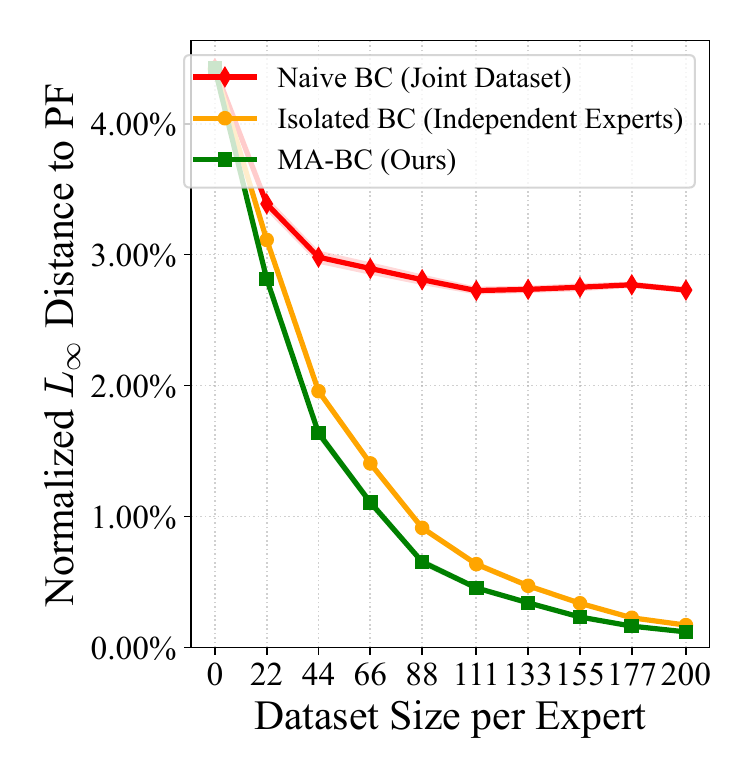}
        \caption{Resource Gathering}
        \label{fig:lc_rg}
    \end{subfigure}

    \caption{Sample efficiency and suboptimality gaps across all three discrete multi-objective environments. The y-axis represents the normalized $L_\infty$ distance to the continuous Pareto frontier. Across all domains, Naive BC (Joint Dataset) fails to converge to an $\varepsilon$-PF policy (\textbf{\textcolor{red!70!black}{Failure II}}). While Isolated BC (Independent Experts) converges to the Pareto Front, it exhibits worse sample efficiency than MA-BC (\textbf{\textcolor{red!70!black}{Failure I}}). Solid lines denote the mean $L_\infty$ distance to the Pareto front, while shaded regions represent the standard error of the mean over 100 trials.}
    \label{fig:learning_curves}
\end{figure*}
We first evaluate the sample-efficiency trade-offs in multi-objective imitation across our discrete benchmarks. Our primary metric is the $L_\infty$ distance to the continuous returns Pareto frontier, computed as $\min_{J^* \in \mathrm{PF}} \| J(\hat{\pi}) - J^* \|_\infty$, where $\hat{\pi}$ is the estimated policy from the data. We apply coordinate-wise normalization to the returns to appropriately adjust for the different scales of each objective.  We compute this distance exactly as described in Appendix \ref{app:linf_lp}. For the multi-output algorithms (Isolated BC and MA-BC), we average the distance across the experts and consolidate the results in Figure \ref{fig:learning_curves}. In each of the environments, we consider data coming from two extreme experts, each optimizing for only one dimension of the reward vector over a uniform initial state distribution. For the Resource Gathering environment, we choose the two experts to be the Gold-optimizing and the Gem-optimizing agents.

Across all environments, the Naive BC approach (which jointly trains on all aggregated data) completely fails to converge to an $\varepsilon$-PF policy as the dataset size increases, which validates Theorem \ref{thm:single_output_BC}, highlighting \textcolor{red!70!black}{\textbf{Failure II}}. The Isolated BC baseline (training independent experts) successfully converges to the Pareto Front. However, it suffers from sample inefficiency as it ignores the underlying structure of the MOMDP (\textcolor{red!70!black}{\textbf{Failure I}}). In stark contrast, MA-BC converges to the Pareto front at a much faster statistical rate by leveraging shared structural knowledge. This confirms the theoretical guarantees of Theorem \ref{thm:multi_BC}.

\subsection{The Impact of Concentrability}

Theorem \ref{thm:lower} suggests that the performance difference between MA-BC and Isolated BC is governed by the concentrability parameter $C_{\pidata}$. Intuitively, this means that the more experts' trajectories overlap, the more sample-efficient MA-BC becomes. To evaluate this relationship directly, our lower bound MOMDP allows us to explicitly compute $C_{\pidata}$. As shown in Figure \ref{fig:concentrability_lower}, as $C_{\pidata}$ decreases, MA-BC actively capitalizes on the shared structure. The resulting empirical gap precisely mirrors our theoretical expectations, validating Theorem \ref{thm:lower}. In order to capture this effect in the Resource Gathering and Deep Sea Treasure environments, we vary the stochasticity of the initial state distribution by controlling the probability of spawning uniformly at random versus starting from a fixed initial state. As illustrated in Figures \ref{fig:rg_alpha_sweep} and \ref{fig:concentrability_dst}, the performance gap widens as the initial state distribution becomes more stochastic. This increased stochasticity causes the experts' trajectories to overlap more frequently, effectively simulating a smaller $C_{\pidata}$.

\begin{figure*}[tbp]
    \centering
    \begin{subfigure}{0.32\textwidth}
        \includegraphics[width=\linewidth]{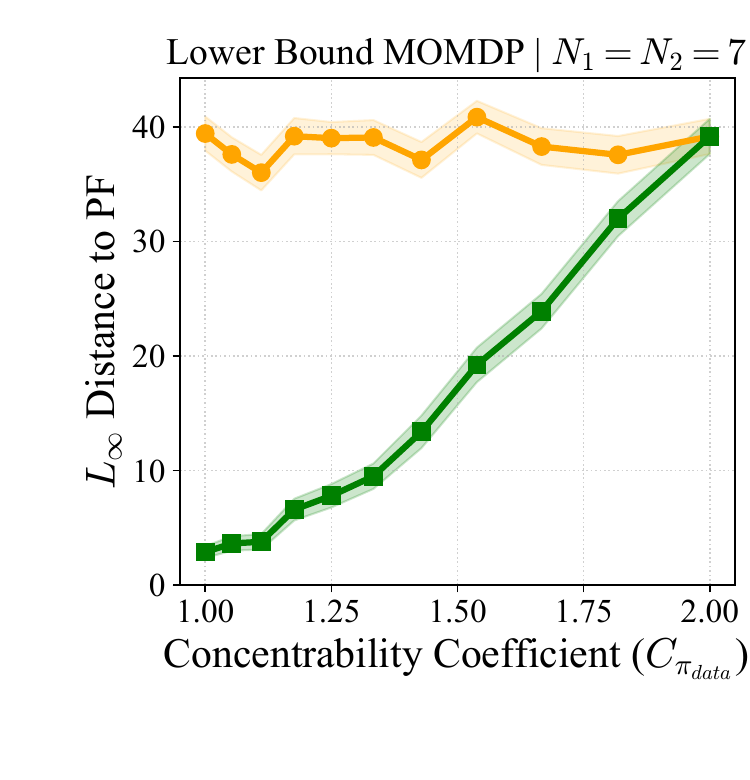}
        \caption{Lower Bound MOMDP}
        \label{fig:concentrability_lower}
    \end{subfigure}\hfill
    \begin{subfigure}{0.32\textwidth}
        \includegraphics[width=\linewidth]{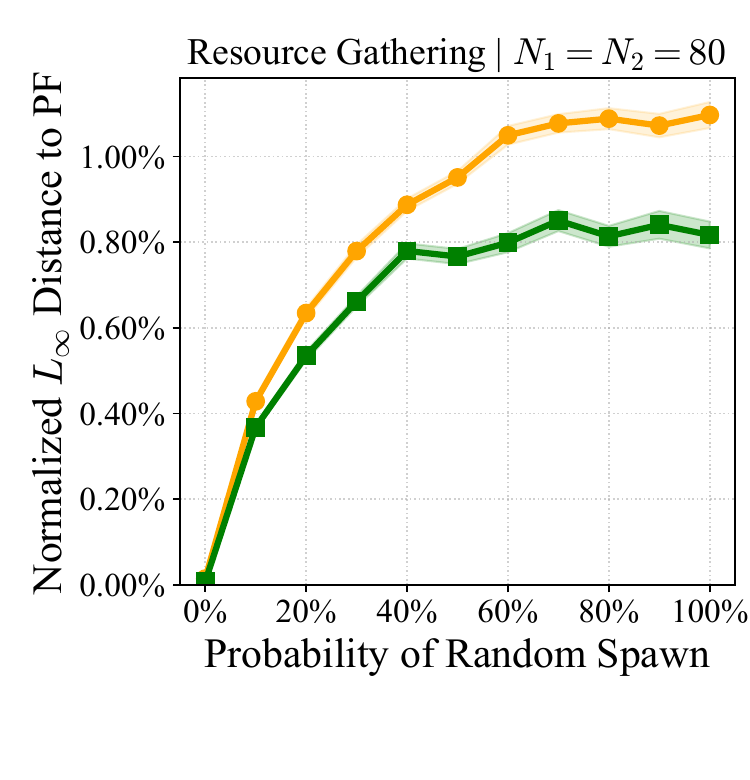}
        \caption{Resource Gathering}
        \label{fig:rg_alpha_sweep}
    \end{subfigure}\hfill
    \begin{subfigure}{0.32\textwidth}
        \includegraphics[width=\linewidth]{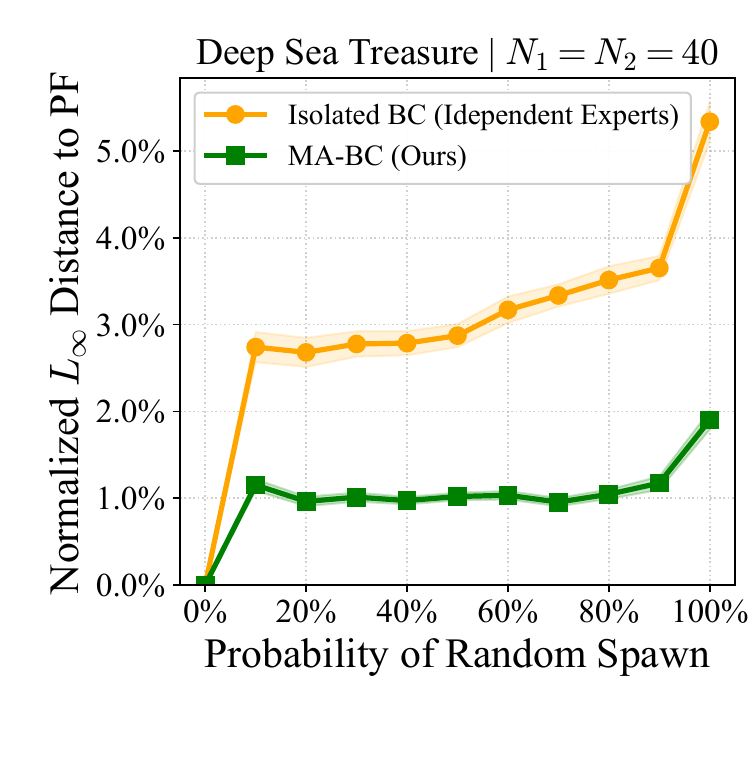}
        \caption{Deep Sea Treasure}
        \label{fig:concentrability_dst}
    \end{subfigure}

    \caption{Evaluation of dataset concentrability on MA-BC performance. \textbf{(Left)} In the Lower Bound MDP, as the concentrability factor decreases, the performance gap between MA-BC and Isolated BC widens. \textbf{(Middle and Right)} We vary the stochasticity of the initial state distribution in the Resource Gathering and the Deep Sea Treasure environments. As the initial state becomes more stochastic, more trajectories overlap between experts, thus providing MA-BC with an advantage. Solid lines denote the mean $L_\infty$ distance to the Pareto front, while shaded regions represent the standard error of the mean over 100 trials.}
    \label{fig:concentrability_analysis}
\end{figure*}

\subsection{Scaling to Continuous State-Action Spaces (LQR Drone Control)}

Finally, we demonstrate that MA-BC scales to continuous domains using a linearized 6-DOF quadcopter navigation task \cite{drone_env}. The multi-objective conflict lies between speed (Agile Expert) and energy conservation (Economic Expert). We estimate the underlying LQR controllers via Ridge Regression ($\lambda = 0.001$) on trajectories sampled from these Pareto-extreme experts (see Appendix \ref{app:drone_lqr_details}). Using our continuous MA-BC extension (Section \ref{sec:continous_mabc}), we apply a spatial tolerance $\delta$ to pool compatible state-action pairs across datasets. We apply an empirical min/max normalization of the states and actions to facilitate tuning the parameter $\delta$. We compute the Lipschitz constants using the true experts' LQR matrices. Figure \ref{fig:drone_eco_summary} evaluates the recovery of the Economic Expert.

\begin{figure*}[tbp]
    \centering
    \includegraphics[width=\linewidth]{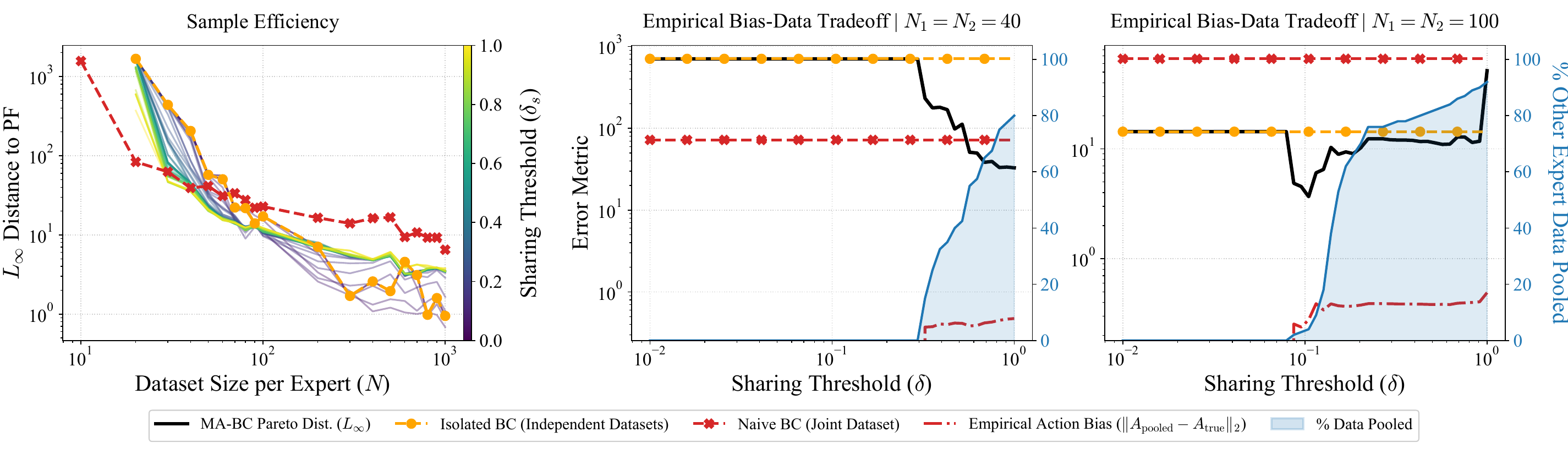}
    \caption{Continuous state-action imitation results for the LQR drone control task. The plots report the performance on recovering the Economic expert for the multi-output learners (i.e., MA-BC and Isolated BC). \textbf{(Left)} Sample efficiency across varying $\delta$ thresholds for MA-BC. \textbf{(Middle \& Right)} The bias-data tradeoff at $N=40$ and $N=100$, demonstrating how the optimal pooling threshold $\delta$ shrinks as the baseline data increases.}
    \label{fig:drone_eco_summary}
\end{figure*}

As shown in the left panel of Figure \ref{fig:drone_eco_summary}, large $\delta$ (yellow lines) drastically improve performance in the extreme low-data regime by aggressively pooling data, whereas small $\delta$ (purple lines) achieve the lowest asymptotic error given sufficient data. This empirical bias-data tradeoff is explicitly quantified in the middle and right panels of Figure \ref{fig:drone_eco_summary}. As $\delta$ increases, MA-BC pools a larger percentage of the other expert's data (blue shaded area), initially reducing the Pareto distance (solid black line). However, excessive $\delta$ injects severe empirical action bias (dashed red line), causing the error to eventually rebound. The empirical action bias is computed as $\norm{A_{\text{pooled}} - A_{\text{true}} }_2$, i.e., the action error introduced by pooling the data. Crucially, comparing the $N=40$ and $N=100$ scenarios reveals a data-dependent ``sweet spot'': as the available data grows, the learner requires less pooling to reduce variance and becomes highly sensitive to bias, causing the optimal $\delta$ threshold to shrink and shift leftwards.

\section{Conclusions}
In this paper, we studied offline imitation learning in MOMDPs for the first time. We identified a key structural property of MOMDPs, the Pareto path existence, which inspired the design of our new algorithm MA-BC. It learns to imitate Pareto Front experts when they are similar while retaining the single expert guarantees when the experts are dissimilar. Moreover, we have proven the optimality of MA-BC across the class of offline imitation algorithms.
Several exciting directions remain open. First, it would be interesting to provide convergence guarantees for our continuous control extension of MA-BC. The smoothness of the Pareto front identified in \cite{jadbabaie2024multi} offers a natural starting point for this analysis. Furthermore, our new structural result in \Cref{thm:pareto_path} enables the development of traversal algorithms for Multi-Objective RL that apply to a larger class of MOMDPs compared to the state-of-the-art in \cite{finding_pareto_front}. Finally, a promising architectural direction involves developing neural network structures that can implicitly learn to pool shared state-action representations while maintaining specialized outputs for divergent expert behaviors.

\newpage
\begin{ack}
We thank Amin Charusaie and Adrian Müller for the valuable discussions throughout the project. C.V. and Z.S. are funded by the Deutsche Forschungsgemeinschaft (DFG, German Research Foundation) under project 468806714 of the Emmy Noether Programme and under Germany's Excellence Strategy – EXC number 2064/1 – Project number 390727645. C.V. also gratefully acknowledges funding from the European Union (ERC, ConSequentIAL, 101165883). Views and opinions expressed are however those of the author(s) only and do not necessarily reflect those of the European Union or the European Research Council. Neither the European Union nor the granting authority can be held responsible for them. C.V. and Z.S. also thank the international Max Planck Research School for Intelligent Systems (IMPRS-IS) for their support. L.V. was supported by the Swiss Data Science Center under fellowship number P22\textunderscore03.  Research by V.C. was sponsored by the Army Research Office under Grant Number W911NF-24-1-0048 and funded by the Swiss National Science Foundation (SNSF) under grant number 2000-1-240094.
\end{ack}

\bibliographystyle{plain}
\bibliography{neurips2026}

\newpage
\appendix
\renewcommand{\contentsname}{Contents of Appendix}
\addtocontents{toc}{\protect\setcounter{tocdepth}{2}}
{
  \setlength{\cftbeforesecskip}{.8em}
  \setlength{\cftbeforesubsecskip}{.5em}
  \tableofcontents
}
\newpage
\section{Related Work}
\label{app:related}

Our work builds on the vast single-objective MDP imitation learning literature  and on the classic and more recent characterizations of the Pareto Front in MOMDPs.
Imitation learning has been a flourishing field in the last three decades, motivated first by robotics \cite{boularias2011relative} and autonomous driving \cite{Pomerleau:1991} and more recently by language modeling \cite{wulfmeier2024imitating,agarwal2024policy}. Common approaches include reductions to supervised learning  \cite{Pomerleau:1991,Ross:2010}, RL in the loop  \cite{Ng:2000,Abbeel:2004,Abbeel:2008,Syed:2007,Syed:2008}, online learning \cite{Shani:2021,viano2024imitation,moulin2025optimistically,swamy2021moments,moulin2025inverseqlearningrightoffline}, regularized RL \cite{Ziebart:2008,Ho:2016,Ho:2016b,viano2022proximal,watson2023coherent} and reward free RL \cite{xu2023provably}.

Our work relies on Behavioral Cloning \cite{Pomerleau:1991}, which is the first imitation learning algorithm. While originally used as a heuristic and long thought to be suboptimal compared to interactive algorithms \cite{Ross:2010,Ross:2011,rajaraman2021value}, it has been proven to be optimal across offline algorithms \cite{rajaraman2020toward}, while \cite{foster2024behavior} establishes that it is actually suboptimal compared to interactive algorithms only when the expert is highly non-stationary.

Secondly, we build on the foundations of multi-objective reinforcement learning, such as early characterizations and existence results of Pareto Front policies \cite{furukawa1980characterization,white1982multi} and on the recent result by \cite{finding_pareto_front}, which established that, when each policy attains a distinct return, neighboring policies on the Pareto front differ in exactly one state. This result inspired our data-sharing idea in the context of imitation learning. Moreover, the Pareto path existence is a direct generalization of their results that allows multiple policies to attain the same returns. In that case, we establish that each vertex on the Pareto front can correspond to multiple policies and that neighboring vertices can always be mapped to policies differing in at most 1 state.
Our result is also related to the recent line of literature \cite{tang2024multiagentimitationlearningvalue,freihaut2025learningequilibriadataprovably,freihaut2025rate,viano2026multi} which also studied imitation from multiple experts, who are, however, assumed to be in an equilibrium profile rather than in a Pareto Front configuration. Next, we review the most closely related works in each area mentioned above.

\paragraph{Classical Imitation Learning }
Imitation Learning has been shown to be successful in autonomous driving applications in \cite{Pomerleau:1991}, which proposed to model the problem as supervised learning, thus inventing behavioral cloning, the algorithm on which we build. Follow-up works focused on the problem of recovering a reward function in addition to reconstructing the expert policy \cite{Ng:2000,Russell:1998,Abbeel:2004,Abbeel:2008,Ziebart:2008}, proving better error propagation analyses via expert interaction \cite{Ross:2010,Ross:2011} or with environment access \cite{Ziebart:2008,boularias2011relative,Ho:2016,Ho:2016b}. We refer to the monograph \cite{Osa:2018} for a complete overview of the classics.

\paragraph{Recent theoretical developments} The first solid sample complexity investigation in imitation learning is by \cite{rajaraman2020toward}, which also introduced the lower bound construction on which we build for ours to show that behavioral cloning is actually optimal across offline algorithms. Beyond the tabular setting, \cite{foster2024behavior} proved that behavioral cloning is suboptimal compared to interactive algorithms only when the expert policy is non-stationary, or without parameter sharing in their terminology. This result shows that the benefit of interaction is more limited than previously thought. On a similar line, we show that MA-BC is optimal across offline algorithms that reconstruct the Pareto Front experts.

\paragraph{Imitation learning in multi-agent systems}
As previously mentioned, approaches that merge all the data as in LLM pretraining \cite{radford2018improving,devlin2019bert} or that treat each expert individually \cite{kim2024pareto} suffer from \textcolor{red!70!black}{\textbf{Failure II}} or \textcolor{red!70!black}{\textbf{Failure I}}, respectively.
Moreover, the only existing approach \cite{franzmeyer2024select} that interpolates between merging all the data and splitting them is still limited by treating all the data collected by an expert as equal. They are either all included or excluded. Therefore, in a simple example with two experts, one of the two failure modes would still apply.
Decentralized imitation learning for learning team behaviors has been studied in \cite{le2017coordinated,song2018multi}.
Moreover, other recent works studied imitation learning in multi-agent systems \cite{yu2019multi,tang2024multiagentimitationlearningvalue,freihaut2025learningequilibriadataprovably,freihaut2025rate,lei2025learning,viano2026multi}; however, taking the different perspective of considering that the expert follows a Nash (or Correlated in \cite{tang2024multiagentimitationlearningvalue}) equilibrium in a Markov Game \cite{shapley1953stochastic}.
These two settings are quite distinct. In their setting each agent receives a scalar reward that depends on other players' strategies, that is the players are coupled. In contrast, in our setting, each agent receives a reward vector which is independent of which other agents are providing demonstrations. The setting of considering both vectorized rewards and coupled agents is perhaps an interesting future direction.

\paragraph{Structural properties of MOMDPs}
The existence of Pareto front policies in any MOMDP has been established by \cite{furukawa1980characterization,white1982multi} in the discounted setting. Moreover, \cite{white1982multi} also suggested an algorithmic approach to reconstruct the state value functions of the Pareto front policies with a value-iteration style approach. Moreover, convexity of Pareto front has been established in \cite{momdp_convexity}.
These papers leave aside the question of how similar the Pareto front policies can be, which is instead settled by \cite{finding_pareto_front} where the authors prove that if each policy has a unique return, then neighboring policies on the Pareto front are guaranteed to differ at only one state. Therefore, there exist very similar Pareto front policies.  In our current work, we establish a more general result saying that even when the same value vector can be attained by multiple policies, we can always find, for each pair of neighboring return vectors, corresponding policies that differ in at most one state.
This structural result also inspired the current approach to reconstruct the full Pareto front in a known MOMDP, which is known as Pareto Front Traversal. The basic idea is to start from a Pareto Front deterministic policy computed via scalarization and some single objective RL algorithm and then flip the policy at one state and check that the new policy (as well as any convex combination with the old one) remains on the Pareto Front. We notice that in the imitation learning context, such an approach is not viable because reward knowledge is needed for the algorithmic routine checking that the new policy is a PF policy.

\newpage
\section{Omitted Proofs}
\label{app:proofs}
We start by presenting the proofs omitted from the main text.
\subsection{Proof of \Cref{thm:pareto_path}}

\Pareto*
\begin{proof}
Let $w$ be the vector for which policies $\pi_A$ and $\pi_B$ are neighboring PF-policies and let $E(J(\pi_A), J(\pi_B)) = \argmax_{J \in \mathbb{J}} w^T J$ be the edge connecting the expected returns vector of policies $\pi_A$ and $\pi_B$ in the return polytope.
    Then, consider the single-objective MDP defined by the scalarized reward function $r_w(x, a) = w^\top r(x, a)$ and let us denote by $\Pi^\star_w$ the set of optimal policies for the scalar reward $r_w$. By the Bellman Optimality Equation, a deterministic policy $\pi$ belongs to $\Pi^*_w$ if it selects an optimal action at every state $x \in \mathcal{X}$: $\pi(x) \in \arg\max_{a \in \mathcal{A}} Q^*_w(x, a)$. 

    Since both $\pi_A$ and $\pi_B$ are in $\Pi^*_w$, both $\pi_A(x)$ and $\pi_B(x)$ are optimal actions for all $x$. By the Markov property, any deterministic policy constructed by independently selecting either $\pi_A(x)$ or $\pi_B(x)$ at each state will also select optimal actions, ensuring it is also in $\Pi^*_w$.

    Suppose $\pi_A$ and $\pi_B$ differ at exactly $k$ states: $\{x_1, \dots, x_k\}$. We construct the sequence $(\pi_0, \dots, \pi_k)$ by initializing $\pi_0 = \pi_A$, and iteratively defining $\pi_i$ by flipping the action at state $x_i$ from $\pi_A(x_i)$ to $\pi_B(x_i)$. By construction, $\pi_k = \pi_B$, and each adjacent pair differs by exactly one state, that is $$\sum_{x\in \X} \mathds{1}{\bc{\pi_{i-1}(x) \neq \pi_{i}(x)}} = 1.$$ Furthermore, since every $\pi_i$ only uses actions from either $\pi_A$ or $\pi_B$, $\pi_i \in \Pi^*_w$ for all $i$. Thus, the return $J(\pi_i)$ remains on the optimal face $E(J(\pi_A), J(\pi_B))$ for the entire sequence. Since this face is the segment between two PF returns, none of its points is dominated (a point dominating a point of the segment would also maximize $w^\top J$, hence lie on the segment, and two points of a segment whose endpoints are mutually non-dominated cannot dominate each other), so every $\pi_i$ is a PF-policy.
\end{proof}

Note the contrast between \Cref{thm:pareto_path} and the result presented in \cite{finding_pareto_front}: Unlike in \cite[Theorem 1]{finding_pareto_front}, we do not assume that neighbors on the Pareto front differ at a single state, since this statement is not true in general (see counter example in \Cref{app:pedagogical_deep_dive}). The breaking of \cite[Theorem 1]{finding_pareto_front} happens in highly structured environments typical of reinforcement learning benchmarks (e.g., grid-worlds with sparse, integer rewards), which can cause one of the following two phenomena to occur:
\begin{enumerate}
    \item \textbf{Aliasing:} Multiple distinct deterministic policies collapse into the exact same geometric vertex of the return polytope.
    \item \textbf{Edge-bound Policies:} Deterministic policies may map to the flat edge connecting two vertices, rather than forming new extreme points.
\end{enumerate}

Because of these phenomena, the strict ``distance-1'' property between adjacent vertices breaks down. Our result in \Cref{thm:pareto_path} holds in general, and is agnostic to the existence of \emph{aliased} or \emph{edge-bound} policies.

\subsection{Proof of \Cref{thm:single_output_BC}}
Here, we present the proof of the failure case for the single-output BC.
\singleBC*
\begin{proof}
    It is enough to consider a $2$-objective MOMDP with a single (absorbing) state and $3$ actions $\A = \bc{a_1, a_2, a_3}$. We consider the following rewards,
    \[
    r(a_1) = [1,0] ~~~~ r(a_2) = [0,1] ~~~~ r(a_3) = [0.5 + \alpha,0.5+\alpha].
    \]
    Now we consider as experts the following policies
    \[
    \pi_1(a_1) = 1 ~~~~~ \pi_2(a_2) = 1,
    \]
    and we consider a dataset containing half of the samples collected from $\pi_1$ and half collected from $\pi_2$, which are both Pareto front policies (i.e., $N_1=N_2$). Since there is a single state, the maximum likelihood policy is the empirical action frequency, so for any dataset size we get the policy $\piout(a_1) = 0.5$,  $\piout(a_2) = 0.5$. Notice now that
    \[
    J_1(\piout) = \frac{0.5}{1-\gamma} ~~~~~ J_2(\piout) = \frac{0.5}{1-\gamma}.
    \]
    However, consider $\pi^\star$ as the policy that deterministically plays $a_3$ we obtain
    \[
    J_1(\pi^\star) = \frac{0.5 + \alpha}{1-\gamma} ~~~~~ J_2(\pi^\star) = \frac{0.5 + \alpha}{1-\gamma} ,
    \]
    for some $\alpha$ such that $0<\alpha<0.5$.
    Therefore, the policy $\piout$ is not on the Pareto front, and its suboptimality gap $\alpha/(1-\gamma)$ grows with $\alpha$. Then, the theorem statement follows from choosing $\alpha = 1/3$.
\end{proof}
\subsection{Proof of \Cref{thm:multi_BC}}
Next, we present the proof of our main result.
\MultiBC*
\begin{proof}
First of all, we notice that the output of \Cref{alg:main} can be rewritten as (see \cite{rajaraman2020toward})
\[
\hat{\pi}_\ell(x) = \begin{cases}
\pi_\ell(x) \quad \text{if}~~~x \in \mathcal{D}_{\ell} \\
\pi_j(x) \quad \text{if}~~~x \in \mathcal{D}_{\mathrm{common}}\setminus \mathcal{D}_{\ell}, \text{ where $j$ is any expert with $x\in\mathcal{D}_j$} \\
1/\abs{\A} \quad \text{otherwise}.
\end{cases}
\]
In the second case, $\pi_j(x)=\pi_\ell(x)$ whenever $x\in\X_{\mathrm{common}}$; for $x\notin\X_{\mathrm{common}}$ the prediction may be wrong, which is accounted for below by charging an error whenever $x\notin\mathcal{D}_\ell$.
The choice of playing actions uniformly at random in unobserved states is arbitrary but that is the choice that allows us to get the best upper bounds as shown in the following.
From this point onwards, without loss of generality let us consider $\ell=1$. Let us introduce the notation $$\norm{J(\pi_1) - J(\hat{\pi}_1)}_{\infty}  = \max_{i \in [d]} \abs{J_i(\pi_1) - J_i(\hat{\pi}_1)}.$$
Then, we can proceed as follows
\begin{align}
    (1-\gamma)^2\norm{J(\pi_1) - J(\hat{\pi}_1)}_{\infty} &\leq \sum_{x\in \X} \nu^{\pi_1}(x) \sum_{a\in \A} \pi_1(a|x) \mathds{1}\bcc{a \neq \hat{\pi}_1(x)} \nonumber\\&=
    \sum_{x\in \X_{\mathrm{common}}} \nu^{\pi_1}(x) \sum_{a\in \A} \pi_1(a|x) \mathds{1}\bcc{a \neq \hat{\pi}_1(x)}\nonumber\\&\phantom{=} + \sum_{x\in \X\setminus \X_{\mathrm{common}}} \nu^{\pi_1}(x) \sum_{a\in \A} \pi_1(a|x) \mathds{1}\bcc{a \neq \hat{\pi}_1(x)}
    \label{eq:decom}
\end{align}
Now, we control each term individually. Starting from the first one:
\begin{align*}
    \sum_{x\in \X_{\mathrm{common}}} &\nu^{\pi_1}(x) \sum_{a\in \A} \pi_1(a|x) \mathds{1}\bcc{a \neq \hat{\pi}_1(x)} \\&= \sum_{x\in \X_{\mathrm{common}}} \nu^{\pi_1}(x) \sum_{a\in \A} \pi_{\mathrm{data}}(a|x) \mathds{1}\bcc{a \neq \hat{\pi}_1(x)} \\
    & \leq \norm{\frac{\nu^{\pi_1}}{\nu^{\pi_{\mathrm{data}}}}}_{\infty}\sum_{x\in \X_{\mathrm{common}}} \nu^{\pi_{\mathrm{data}}}(x) \sum_{a\in \A} \pi_{\mathrm{data}}(a|x) \mathds{1}\bcc{a \neq \hat{\pi}_1(x)},
\end{align*}
where for the first equality in the block above we have used that all experts agree on $\X_{\mathrm{common}}$, so that we can define $\pi_{\mathrm{data}}(x) := \pi_1(x) = \dots = \pi_L(x)$ for all $x\in\X_{\mathrm{common}}$.
Then, taking expectation over the randomness involved in the collection of $\mathcal{D}_{\mathrm{common}}$, and using that for $x\in \X_{\mathrm{common}}$ the prediction of $\hat{\pi}_1$ is correct as soon as $x$ is sampled once, i.e., as soon as it appears in $\mathcal{D}_{\mathrm{common}}$, no matter which expert sampled the action at that state, we have that
\begin{align*}
&\mathbb{E}\bs{\sum_{x\in \X_{\mathrm{common}}} \nu^{\pi_1}(x) \sum_{a\in \A} \pi_1(a|x) \mathds{1}\bcc{a \neq \hat{\pi}_1(x)}} \\ &\leq \norm{\frac{\nu^{\pi_1}}{\nu^{\pi_{\mathrm{data}}}}}_{\infty,\X_{\mathrm{common}}}\sum_{x\in \X_{\mathrm{common}}} \nu^{\pi_{\mathrm{data}}}(x) \mathbb{P}\bs{x \notin \mathcal{D}_{\mathrm{common}}} \\
&\leq \norm{\frac{\nu^{\pi_1}}{\nu^{\pi_{\mathrm{data}}}}}_{\infty,\X_{\mathrm{common}}}\sum_{x\in \X_{\mathrm{common}}} \nu^{\pi_{\mathrm{data}}}(x) (1-\nu^{\pi_{\mathrm{data}}}(x))^N \\
&\leq \norm{\frac{\nu^{\pi_1}}{\nu^{\pi_{\mathrm{data}}}}}_{\infty,\X_{\mathrm{common}}}\sum_{x\in \X_{\mathrm{common}}} \frac{1}{e N} \\
&\leq \norm{\frac{\nu^{\pi_1}}{\nu^{\pi_{\mathrm{data}}}}}_{\infty,\X_{\mathrm{common}}}\frac{\abs{\X_{\mathrm{common}}} }{e N}.
\end{align*}
The second step uses $\mathbb{P}\bs{x \notin \mathcal{D}_{\mathrm{common}}} = \prod_{j=1}^L (1 - \nu^{\pi_j}(x))^{N_j} \leq (1 - \nu^{\pi_{\mathrm{data}}}(x))^N$, which follows from Jensen's inequality applied to the concave function $\log(1-\cdot)$ with weights $N_j/N$ (the two sides need not be equal). The third step uses $p(1-p)^N \leq p e^{-pN} \leq 1/(eN)$.
This concludes the bound on the first term of \eqref{eq:decom}. Now let us look at the second term.
\begin{align*}
    \mathbb{E}\bs{\sum_{x\in \X\setminus \X_{\mathrm{common}}} \nu^{\pi_1}(x) \sum_{a\in \A} \pi_1(a|x) \mathds{1}\bcc{a \neq \hat{\pi}_1(x)}} & \leq \mathbb{E}\bs{\sum_{x\in \X\setminus \X_{\mathrm{common}}} \nu^{\pi_1}(x) \mathbb{P}\bs{x \notin \mathcal{D}_1}} \\
    &=\sum_{x\in \X\setminus \X_{\mathrm{common}}} \nu^{\pi_1}(x) (1-\nu^{\pi_1}(x))^{N_1} \\
    &\leq \frac{\abs{\X\setminus\X_{\mathrm{common}}}}{e N_1}.
\end{align*}
Therefore, in conclusion, we get the bound
\[
(1-\gamma)^2\mathbb{E}\bs{\norm{J(\pi_1) - J(\hat{\pi}_1)}_{\infty}} \leq \norm{\frac{\nu^{\pi_1}}{\nu^{\pi_{\mathrm{data}}}}}_{\infty,\X_{\mathrm{common}}}\frac{\abs{\X_{\mathrm{common}}} }{e N} + \frac{\abs{\X\setminus\X_{\mathrm{common}}}}{e N_1}.
\]
Because the $L_\infty$ norm bounds the maximum deviation across all dimensions (and $\abs{J_i(\pi_1)-\mathbb{E}[J_i(\hat{\pi}_1)]}\leq \mathbb{E}\abs{J_i(\pi_1)-J_i(\hat{\pi}_1)}$ by Jensen's inequality), our BC analysis guarantees that for all $i \in [d]$:

\[ (1-\gamma)^2\mathbb{E}[J_i(\hat{\pi}_1)] + \left\|\frac{\nu^{\pi_1}}{\nu^{\pi_{\mathrm{data}}}}\right\|_{\infty,\X_{\mathrm{common}}} \frac{|\mathcal{X}_{\mathrm{common}}|}{eN} + \frac{|\mathcal{X}\setminus\mathcal{X}_{\mathrm{common}}|}{eN_1} \ge (1-\gamma)^2 J_i(\pi_1).
\]

Now, using the fact that $\pi_1$ is on the Pareto front, we know there does not exist any policy $\pi^*$ that dominates it. Therefore, substituting our bound into the definition of the Pareto front, it holds that $\nexists \pi^*$ such that both of the following conditions are true simultaneously:

\[ \forall i \in [d], ~~ J_i(\pi^*) \ge \mathbb{E}[J_i(\hat{\pi}_1)] + \left\|\frac{\nu^{\pi_1}}{\nu^{\pi_{\mathrm{data}}}}\right\|_{\infty,\X_{\mathrm{common}}} \frac{|\mathcal{X}_{\mathrm{common}}|}{e(1-\gamma)^2N} + \frac{|\mathcal{X}\setminus\mathcal{X}_{\mathrm{common}}|}{e(1-\gamma)^2 N_1} \]

\[ \exists j \in [d], ~~ J_j(\pi^*) > \mathbb{E}[J_j(\hat{\pi}_1)] + \left\|\frac{\nu^{\pi_1}}{\nu^{\pi_{\mathrm{data}}}}\right\|_{\infty,\X_{\mathrm{common}}} \frac{|\mathcal{X}_{\mathrm{common}}|}{e(1-\gamma)^2N} + \frac{|\mathcal{X}\setminus\mathcal{X}_{\mathrm{common}}|}{e(1-\gamma)^2 N_1}. \]
\end{proof}

\subsection{Proof of \Cref{thm:lower}}
\label{app:proof_lower_bound}
Before presenting the proof we formalize the algorithmic class under which the lower bound holds.

\paragraph{Expert-indexed output algorithms.} An expert-indexed multi-output imitation-learning algorithm returns one
policy $\widehat{\pi}_\ell$ associated with each demonstrated expert
$\pi_\ell$. In the construction below, this association is enforced by
requiring $\widehat{\pi}_\ell$ to use the same root action as
$\pi_\ell$. The corresponding root action is revealed to the learner
for free, and therefore does not contribute to the statistical
difficulty of the instance.

\L*
\paragraph{Proof intuition.}
We adapt the hard-instance idea of \cite{rajaraman2020toward} using
three recurrent reset blocks: one common block and two private blocks,
as illustrated in \Cref{fig:lower_bound_routing}. Expert $\ell$ enters
the common block with probability $p_\ell$ and its own private block
with probability $1-p_\ell$. Only its own demonstrations reveal the
correct actions in its private block, whereas demonstrations from both
experts reveal the correct actions in the common block.

Because the observations are sampled from normalized discounted
occupancy measures, the expected numbers of samples in these blocks are
$\gamma N_\ell(1-p_\ell)$ and
$\gamma M_{\mathrm{common}}$, respectively, where
$M_{\mathrm{common}}:=N_1p_1+N_2p_2$.
The effect of an error in the common block on expert $\ell$'s return
is also weighted by $p_\ell$. The exact common-state concentrability
coefficient is therefore
\[
C_{\pi_\ell}=\frac{Np_\ell}{M_{\mathrm{common}}}.
\]
The recurrent reset makes an unobserved action a repeated failure
hazard over the discounted horizon. In the regime before the return
loss saturates, the resulting common and private contributions scale as
\[
\frac{C_{\pi_\ell}\abs{\X_{\mathrm{common}}}}
     {N(1-\gamma)^2}
\qquad\text{and}\qquad
\frac{\abs{\X\setminus\X_{\mathrm{common}}}}
     {N_\ell(1-\gamma)^2}.
\]
The formal argument below retains the saturation terms in
\Cref{thm:lower} and proves a lower bound for each fixed expert-indexed
output. Two experts and two objectives suffice because this is a
special case of the general problem class.

\paragraph{Construction of the lower-bound MOMDP.}
Fix $p_1,p_2\in(0,1/2]$ and define
\[
M_{\mathrm{common}}
:=
N_1p_1+N_2p_2.
\]
Let $\beta:=1/4$. The state space is
\[
\X
=
\{x_{\mathrm{root}},x_{\mathrm{end}}\}
\cup \X^{\mathrm c}
\cup \X_{\mathrm{div}}^1
\cup \X_{\mathrm{div}}^2,
\]
where $\X^{\mathrm c}$ is the common block and
$\X_{\mathrm{div}}^\ell$ is the private block associated with expert
$\ell$. Each block consists of a collection of hard states and one
anchor state. The correct action at the anchor state is fixed and
revealed to the learner; all statistical uncertainty is concentrated
on the hard states.

The action set is $\A=\{a_1,a_2\}$. At the root state, the reward is
zero and the transition kernel satisfies
\begin{align}
P(\cdot\mid x_{\mathrm{root}},a_1)
&=
p_1\rho_{\mathrm c}
+
(1-p_1)\rho_1,
\\
P(\cdot\mid x_{\mathrm{root}},a_2)
&=
p_2\rho_{\mathrm c}
+
(1-p_2)\rho_2,
\end{align}
where $\rho_{\mathrm c}$, $\rho_1$, and $\rho_2$ are distributions
supported on $\X^{\mathrm c}$,
$\X_{\mathrm{div}}^1$, and
$\X_{\mathrm{div}}^2$, respectively.

Every hard state $x$ is assigned a correct action
$a^\star(x)\in\{a_1,a_2\}$. Under the Yao prior, these actions are
independent and uniformly distributed. For a state $x$ in a block
$B\in\{\mathrm c,1,2\}$, define
\begin{align}
r(x,a^\star(x))
&=
u_B,
&
P(\cdot\mid x,a^\star(x))
&=
\rho_B,
\\
r(x,a)
&=
\mathbf{0},
\qquad a\neq a^\star(x),
&
P(x_{\mathrm{end}}\mid x,a)
&=
1.
\end{align}
The same transitions hold at the anchor state, except that its correct
action is known to the learner. The reward vectors are
\[
u_{\mathrm c}
=
\begin{bmatrix}
\beta\\
\beta
\end{bmatrix},
\qquad
u_1
=
\begin{bmatrix}
1\\
\beta
\end{bmatrix},
\qquad
u_2
=
\begin{bmatrix}
\beta\\
1
\end{bmatrix}.
\]
Finally, $x_{\mathrm{end}}$ is absorbing and has zero reward.

The two expert policies satisfy
\[
\pi_1(x_{\mathrm{root}})=a_1,
\qquad
\pi_2(x_{\mathrm{root}})=a_2.
\]
Both experts take $a^\star(x)$ on the common block. Expert $\ell$
takes $a^\star(x)$ on its own private block
$\X_{\mathrm{div}}^\ell$. On the private block that it never visits,
each expert is defined to take the action opposite to $a^\star(x)$.
The two policies take the same arbitrary action at
$x_{\mathrm{end}}$. Consequently,
\[
\X_{\mathrm{common}}
=
\X^{\mathrm c}\cup\{x_{\mathrm{end}}\},
\]
whereas $x_{\mathrm{root}}$ and the two private blocks belong to
$\X\setminus\X_{\mathrm{common}}$.

The essential feature of the construction is the reset mechanism:
after every correct action, the process returns to a fresh state drawn
from the same block distribution. Thus, uncertainty about an
unobserved action creates a failure hazard throughout the effective
discounted horizon rather than at a single time step.

\begin{lemma}[Value gap in a recurrent reset block]
\label{lem:reset_value}
Consider a block $B$ with reset distribution $\rho_B$ and reward
vector $u_B$. Let $\pi^\star$ take $a^\star(x)$ at every $x\in B$.
For an arbitrary policy $\widehat{\pi}$, define its failure probability
under a fresh draw from the block as
\[
q_B(\widehat{\pi})
:=
\sum_{x\in B}
\rho_B(x)
\left(
1-\widehat{\pi}(a^\star(x)\mid x)
\right).
\]
Then, for every objective $i\in[2]$,
\begin{align}
V_i^{\pi^\star}(B)
&=
\frac{u_{B,i}}{1-\gamma},
\\
V_i^{\widehat{\pi}}(B)
&=
\frac{
(1-q_B(\widehat{\pi}))u_{B,i}
}{
1-\gamma(1-q_B(\widehat{\pi}))
},
\end{align}
and hence
\begin{equation}
\label{eq:reset_exact_gap}
V_i^{\pi^\star}(B)
-
V_i^{\widehat{\pi}}(B)
=
\frac{
u_{B,i}q_B(\widehat{\pi})
}{
(1-\gamma)
\left(
1-\gamma+\gamma q_B(\widehat{\pi})
\right)
}.
\end{equation}
In particular, if
$\gamma q_B(\widehat{\pi})\leq c_0(1-\gamma)$, then
\[
V_i^{\pi^\star}(B)
-
V_i^{\widehat{\pi}}(B)
\geq
\frac{u_{B,i}}{1+c_0}
\frac{q_B(\widehat{\pi})}{(1-\gamma)^2}.
\]
\end{lemma}

\begin{proof}
After every successful action, the next state is an independent draw
from $\rho_B$. Therefore, at every live time step, the policy succeeds
with probability $1-q_B(\widehat{\pi})$, receives reward $u_{B,i}$,
and otherwise moves to $x_{\mathrm{end}}$. It follows that
\[
V_i^{\widehat{\pi}}(B)
=
(1-q_B(\widehat{\pi}))
\left(
u_{B,i}
+
\gamma V_i^{\widehat{\pi}}(B)
\right).
\]
Solving this equation gives the claimed expression for
$V_i^{\widehat{\pi}}(B)$. Subtracting it from
$u_{B,i}/(1-\gamma)$ gives \Cref{eq:reset_exact_gap}.

Finally, if
$\gamma q_B(\widehat{\pi})\leq c_0(1-\gamma)$, then
\[
1-\gamma+\gamma q_B(\widehat{\pi})
\leq
(1+c_0)(1-\gamma),
\]
which proves the final claim.
\end{proof}

\begin{proof}[Proof of \Cref{thm:lower}]
Fix an arbitrary index $\ell\in\{1,2\}$. The hard instance may depend
on this fixed index, as is standard for a per-output minimax lower
bound.

\paragraph{Step 1: the expert policies lie on the Pareto front.}
Consider first a policy that chooses $a_1$ at
$x_{\mathrm{root}}$. Replacing its actions on
$\X^{\mathrm c}\cup\X_{\mathrm{div}}^1$ by the correct actions
$a^\star(x)$ weakly improves both return coordinates and strictly
improves both whenever the original policy has a positive failure
probability on a reachable block. Therefore, among policies choosing
$a_1$ at the root, $\pi_1$ is the unique nondominated policy on the
reachable states. The analogous statement holds for $\pi_2$ among
policies choosing $a_2$ at the root.

The expert returns are
\begin{align}
J(\pi_1)
&=
\frac{\gamma}{1-\gamma}
\left(
p_1u_{\mathrm c}
+
(1-p_1)u_1
\right),
\\
J(\pi_2)
&=
\frac{\gamma}{1-\gamma}
\left(
p_2u_{\mathrm c}
+
(1-p_2)u_2
\right).
\end{align}
Since $\beta=1/4$ and $p_1,p_2\leq1/2$, the first coordinate of
$J(\pi_1)$ is strictly larger than the first coordinate of
$J(\pi_2)$, whereas the second coordinate of $J(\pi_2)$ is strictly
larger than the second coordinate of $J(\pi_1)$. Hence neither expert
dominates the other.

More generally, take any policy and preserve its randomization at the
root while replacing all of its reachable block actions by the
corresponding correct actions. The resulting policy has a return in
$\operatorname{conv}\{J(\pi_1),J(\pi_2)\}$ and weakly dominates the
original policy. Consequently, $\pi_1$ and $\pi_2$ are deterministic
PF-policies.

\paragraph{Step 2: occupancy measures and common-state concentrability.}
Under expert $j$, the normalized discounted occupancy of
$x\in\X^{\mathrm c}$ is
\begin{align}
\nu^{\pi_j}(x)
&=
(1-\gamma)
\sum_{t=1}^{\infty}
\gamma^t p_j\rho_{\mathrm c}(x)
\\
&=
\gamma p_j\rho_{\mathrm c}(x).
\end{align}
It follows that
\[
\nu^{\pi_{\mathrm{data}}}(x)
=
\gamma\rho_{\mathrm c}(x)
\frac{N_1p_1+N_2p_2}{N}.
\]
Therefore, for every common state with positive pooled occupancy,
\[
\frac{
\nu^{\pi_\ell}(x)
}{
\nu^{\pi_{\mathrm{data}}}(x)
}
=
\frac{
Np_\ell
}{
N_1p_1+N_2p_2
}.
\]
The ratio is constant over the common block, and hence
\begin{equation}
\label{eq:lower_concentrability}
C_{\pi_\ell}
=
\norm{
\frac{\nu^{\pi_\ell}}
     {\nu^{\pi_{\mathrm{data}}}}
}_{\infty,\X_{\mathrm{common}}}
=
\frac{
Np_\ell
}{
N_1p_1+N_2p_2
}.
\end{equation}
Notice that on a private state
$x\in\X_{\mathrm{div}}^\ell$, the corresponding full-space occupancy
ratio is $N/N_\ell$. This is why the pooled term uses a
concentrability coefficient restricted to $\X_{\mathrm{common}}$,
whereas the private term is controlled directly by $N_\ell$.

\paragraph{Step 3: choice of the rare-state probabilities.}
Let $n_{\mathrm c}$ denote the number of hard states in
$\X^{\mathrm c}$ and let $n_{\mathrm d}$ denote the number of hard
states in each private block. Since each block contains only one
additional anchor state,
\[
n_{\mathrm c}
=
\Theta\left(\abs{\X_{\mathrm{common}}}\right),
\qquad
n_{\mathrm d}
=
\Theta\left(
\abs{\X\setminus\X_{\mathrm{common}}}
\right),
\]
where the hidden constants are universal.

For a sufficiently small universal constant
$c_0\in(0,1/16]$, define
\begin{align}
\eta_{\mathrm c}
&:=
c_0
\min\bcc{
\frac{1}{
\gamma M_{\mathrm{common}}
},
\frac{1-\gamma}{
\gamma n_{\mathrm c}
}
},
\\
\eta_\ell
&:=
c_0
\min\bcc{
\frac{1}{
\gamma(1-p_\ell)N_\ell
},
\frac{1-\gamma}{
\gamma n_{\mathrm d}
}
}.
\end{align}
Assign probability $\eta_{\mathrm c}$ to each hard state under
$\rho_{\mathrm c}$ and probability $\eta_\ell$ to each hard state in
$\X_{\mathrm{div}}^\ell$ under $\rho_\ell$. Place all remaining
probability on the corresponding anchor state. The analogous
distribution on the other private block can be defined symmetrically.

These are valid probability distributions because the total
probability assigned to the hard states is at most
$c_0(1-\gamma)/\gamma<1$. Moreover,
\begin{equation}
\label{eq:small_hazard}
\gamma n_{\mathrm c}\eta_{\mathrm c}
\leq
c_0(1-\gamma),
\qquad
\gamma n_{\mathrm d}\eta_\ell
\leq
c_0(1-\gamma).
\end{equation}

\paragraph{Step 4: unobserved states imply unavoidable action error.}
For each block $B$, let
$\overline{q}_B(\widehat{\pi}_\ell)$ denote the contribution of only
the hard states to $q_B(\widehat{\pi}_\ell)$. Then
\[
q_B(\widehat{\pi}_\ell)
\geq
\overline{q}_B(\widehat{\pi}_\ell),
\]
and, by \Cref{eq:small_hazard},
\[
\gamma\overline{q}_B(\widehat{\pi}_\ell)
\leq
c_0(1-\gamma),
\]
regardless of the learner's action at the anchor state.

Fix a hard state $x\in\X^{\mathrm c}$. Its normalized occupancy under
expert $j$ is $\gamma p_j\eta_{\mathrm c}$. Since
$\mathcal{D}_j$ contains $N_j$ independent samples from
$\mu^{\pi_j}$, the probability that $x$ is absent from all datasets is
\begin{align}
\mathbb{P}\left[
x\notin\bigcup_{j=1}^2\mathcal{D}_j
\right]
&=
\prod_{j=1}^2
\left(
1-\gamma p_j\eta_{\mathrm c}
\right)^{N_j}
\\
&\geq
1-
\gamma\eta_{\mathrm c}
\left(
N_1p_1+N_2p_2
\right)
\\
&\geq
1-c_0.
\label{eq:unseen_common_lower}
\end{align}
The first inequality follows by upper-bounding the probability that
the state appears at least once by the sum of its appearance
probabilities.

Conditional on $x$ being absent, its hidden action remains uniform on
$\{a_1,a_2\}$ under the Yao prior. Therefore, every learner satisfies
\[
\mathbb{E}\left[
1-
\widehat{\pi}_\ell(a^\star(x)\mid x)
\;\middle|\;
x\notin\bigcup_{j=1}^2\mathcal{D}_j
\right]
\geq
\frac{1}{2}.
\]
Summing over the hard common states gives
\begin{equation}
\label{eq:common_failure_lower}
\mathbb{E}\left[
\overline{q}_{\mathrm c}(\widehat{\pi}_\ell)
\right]
\geq
\frac{1-c_0}{2}
n_{\mathrm c}\eta_{\mathrm c}.
\end{equation}

Similarly, a hard state
$x\in\X_{\mathrm{div}}^\ell$ is visited only by expert $\ell$ and has
occupancy $\gamma(1-p_\ell)\eta_\ell$. Hence
\begin{align}
\mathbb{P}\left[
x\notin\mathcal{D}_\ell
\right]
&=
\left(
1-
\gamma(1-p_\ell)\eta_\ell
\right)^{N_\ell}
\\
&\geq
1-c_0.
\end{align}
The demonstrations of the other expert contain no information about
the hidden actions in $\X_{\mathrm{div}}^\ell$. Consequently,
\begin{equation}
\label{eq:private_failure_lower}
\mathbb{E}\left[
\overline{q}_\ell(\widehat{\pi}_\ell)
\right]
\geq
\frac{1-c_0}{2}
n_{\mathrm d}\eta_\ell.
\end{equation}

\paragraph{Step 5: conversion of action error into return loss.}
The expert-indexing condition fixes the root action of
$\widehat{\pi}_\ell$ to be $a_\ell$. Define
\[
f_B(q)
:=
\frac{
u_{B,i}q
}{
(1-\gamma)(1-\gamma+\gamma q)
}.
\]
This function is nondecreasing in $q$. Since
$q_B(\widehat{\pi}_\ell)
\geq
\overline{q}_B(\widehat{\pi}_\ell)$, the exact expression in
\Cref{eq:reset_exact_gap} gives
\[
V_i^{\pi^\star}(B)
-
V_i^{\widehat{\pi}_\ell}(B)
\geq
f_B\left(
\overline{q}_B(\widehat{\pi}_\ell)
\right).
\]
The small-hazard condition in \Cref{eq:small_hazard} can therefore be
applied to $\overline{q}_B(\widehat{\pi}_\ell)$. For every objective
$i\in[2]$, we obtain
\begin{align}
J_i(\pi_\ell)
-
\mathbb{E}\left[
J_i(\widehat{\pi}_\ell)
\right]
&\geq
\frac{\gamma}{1+c_0}
\frac{1}{(1-\gamma)^2}
\Big(
p_\ell u_{\mathrm c,i}
\mathbb{E}\left[
\overline{q}_{\mathrm c}(\widehat{\pi}_\ell)
\right]
\nonumber\\
&\hspace{4.7cm}
+
(1-p_\ell)u_{\ell,i}
\mathbb{E}\left[
\overline{q}_{\ell}(\widehat{\pi}_\ell)
\right]
\Big)
\\
&\geq
c_1
\frac{\gamma}{(1-\gamma)^2}
\left(
p_\ell n_{\mathrm c}\eta_{\mathrm c}
+
(1-p_\ell)n_{\mathrm d}\eta_\ell
\right),
\label{eq:coordinate_lower_gap}
\end{align}
where $c_1>0$ is a universal constant, since every coordinate of
$u_{\mathrm c}$, $u_1$, and $u_2$ is at least $\beta=1/4$.

Substituting the definitions of $\eta_{\mathrm c}$ and $\eta_\ell$
and using $\gamma\geq1/2$ gives
\begin{align}
J_i(\pi_\ell)
-
\mathbb{E}\left[
J_i(\widehat{\pi}_\ell)
\right]
\geq
\Omega\Bigg(
&
\min\bcc{
\frac{p_\ell}{1-\gamma},
\frac{
p_\ell n_{\mathrm c}
}{
M_{\mathrm{common}}(1-\gamma)^2
}
}
\nonumber\\
&+
\min\bcc{
\frac{1-p_\ell}{1-\gamma},
\frac{
n_{\mathrm d}
}{
N_\ell(1-\gamma)^2
}
}
\Bigg).
\label{eq:pre_final_lower_bound}
\end{align}
By \Cref{eq:lower_concentrability},
\[
\frac{p_\ell}{M_{\mathrm{common}}}
=
\frac{C_{\pi_\ell}}{N}.
\]
Replacing $n_{\mathrm c}$ and $n_{\mathrm d}$ by the corresponding
state-set cardinalities up to universal constants yields
\[
J_i(\pi_\ell)
-
\mathbb{E}\left[
J_i(\widehat{\pi}_\ell)
\right]
\geq
\Omega\br{
\min\bcc{
\frac{p_\ell}{1-\gamma},
C_{\pi_\ell}
\frac{
\abs{\X_{\mathrm{common}}}
}{
N(1-\gamma)^2
}
}
+
\min\bcc{
\frac{1-p_\ell}{1-\gamma},
\frac{
\abs{\X\setminus\X_{\mathrm{common}}}
}{
N_\ell(1-\gamma)^2
}
}
}.
\]

\paragraph{Step 6: conversion to a deterministic hard instance and
Pareto-front error.}
The preceding argument gives a Bayes lower bound under the independent
uniform prior over the hidden actions. By Yao's minimax principle,
there exists a deterministic realization of the hidden actions for
which the same lower bound, up to a universal constant, holds for the
fixed output index $\ell$.

Moreover, \Cref{eq:coordinate_lower_gap} holds for both objectives.
Thus, the expert policy $\pi_\ell$ improves every coordinate of
$\mathbb{E}[J(\widehat{\pi}_\ell)]$ by the displayed amount.
According to the definition of an $\varepsilon_\ell$-PF-policy,
$\widehat{\pi}_\ell$ cannot be an $\varepsilon_\ell$-PF-policy for any
smaller value of $\varepsilon_\ell$. This proves
\Cref{thm:lower}.
\end{proof}

\paragraph{Why the reset mechanism recovers the matching rate.}
In the previous one-shot construction, an error at a hard state could
occur only once, which lost one factor of $(1-\gamma)^{-1}$. In the
construction above, every correct action returns the process to a
fresh state from the same hard distribution. The learner therefore
faces an approximately constant failure hazard at every step of the
effective discounted horizon. The denominator in
\Cref{eq:reset_exact_gap} consequently produces the additional factor
of $(1-\gamma)^{-1}$, yielding the matching
$(1-\gamma)^{-2}$ dependence in both the common-state and
expert-specific terms.

\begin{figure*}[t]
    \centering

    \begin{subfigure}{0.96\textwidth}
        \centering
        \begin{tikzpicture}[
            >=Latex,
            font=\small,
            root/.style={
                circle,
                draw=black!75,
                thick,
                fill=gray!8,
                minimum size=14mm
            },
            actionone/.style={
                circle,
                draw=blue!70!black,
                thick,
                fill=blue!5,
                text=blue!70!black,
                minimum size=11mm
            },
            actiontwo/.style={
                circle,
                draw=orange!85!black,
                thick,
                fill=orange!7,
                text=orange!85!black,
                minimum size=11mm
            },
            commonblock/.style={
                rounded corners=4pt,
                draw=teal!70!black,
                thick,
                fill=teal!5,
                minimum width=5.0cm,
                minimum height=1.45cm,
                align=center
            },
            privateone/.style={
                rounded corners=4pt,
                draw=blue!70!black,
                thick,
                fill=blue!5,
                minimum width=5.0cm,
                minimum height=1.45cm,
                align=center
            },
            privatetwo/.style={
                rounded corners=4pt,
                draw=orange!85!black,
                thick,
                fill=orange!7,
                minimum width=5.0cm,
                minimum height=1.45cm,
                align=center
            },
            routeone/.style={
                ->,
                thick,
                blue!70!black
            },
            routetwo/.style={
                ->,
                thick,
                orange!85!black
            }
        ]
            \node[root] (root) at (0,0) {$x_{\mathrm{root}}$};
            \node[actionone] (a1) at (2.7,1.55) {$a_1$};
            \node[actiontwo] (a2) at (2.7,-1.55) {$a_2$};
            \node[privateone] (private1) at (7.4,2.5)
            {
                $\X_{\mathrm{div}}^1$
                \\[0.5mm]
                \textbf{Expert 1 private block}
                \\[0.5mm]
                \scriptsize only $\mathcal{D}_1$ is informative
            };
            \node[commonblock] (common) at (7.4,0)
            {
                $\X^{\mathrm c}\subseteq\X_{\mathrm{common}}$
                \\[0.5mm]
                \textbf{Common block}
                \\[0.5mm]
                \scriptsize demonstrations can be pooled
            };
            \node[privatetwo] (private2) at (7.4,-2.5)
            {
                $\X_{\mathrm{div}}^2$
                \\[0.5mm]
                \textbf{Expert 2 private block}
                \\[0.5mm]
                \scriptsize only $\mathcal{D}_2$ is informative
            };
            \draw[routeone]
                (root) to[bend left=16]
                node[midway, above left]
                {$\pi_1(x_{\mathrm{root}})=a_1$}
                (a1);
            \draw[routetwo]
                (root) to[bend right=16]
                node[midway, below left]
                {$\pi_2(x_{\mathrm{root}})=a_2$}
                (a2);
            \draw[routeone]
                (a1) to[bend left=10]
                node[midway, above, fill=white, inner sep=1.5pt]
                {$1-p_1$}
                (private1.west);
            \draw[routeone]
                (a1) to[bend right=7]
                node[midway, below, fill=white, inner sep=1.5pt]
                {$p_1$}
                (common.west);
            \draw[routetwo]
                (a2) to[bend left=7]
                node[midway, above, fill=white, inner sep=1.5pt]
                {$p_2$}
                (common.west);
            \draw[routetwo]
                (a2) to[bend right=10]
                node[midway, below, fill=white, inner sep=1.5pt]
                {$1-p_2$}
                (private2.west);
            \node[
                anchor=west,
                align=left,
                font=\scriptsize,
                text=blue!70!black
            ] at (10.4,2.5)
            {effective sample size $N_1$};
            \node[
                anchor=west,
                align=left,
                font=\scriptsize,
                text=teal!70!black
            ] at (10.4,0)
            {effective sample size $N$};
            \node[
                anchor=west,
                align=left,
                font=\scriptsize,
                text=orange!85!black
            ] at (10.4,-2.5)
            {effective sample size $N_2$};
        \end{tikzpicture}

        \caption{Root routing and the common--divergent decomposition.}
        \label{fig:lower_bound_routing}
    \end{subfigure}

    \vspace{0.4cm}

    \begin{subfigure}{0.96\textwidth}
        \centering
        \begin{tikzpicture}[
            >=Latex,
            font=\small,
            distribution/.style={
                rounded corners=4pt,
                draw=black!65,
                thick,
                fill=gray!5,
                minimum width=3.6cm,
                minimum height=1.35cm,
                align=center
            },
            state/.style={
                circle,
                draw=black!75,
                thick,
                fill=white,
                minimum size=12mm
            },
            success/.style={
                rounded corners=4pt,
                draw=teal!70!black,
                thick,
                fill=teal!5,
                minimum width=3.9cm,
                minimum height=1.35cm,
                align=center
            },
            terminal/.style={
                circle,
                draw=red!75!black,
                thick,
                fill=red!5,
                minimum size=12mm
            },
            goodedge/.style={
                ->,
                very thick,
                teal!70!black
            },
            badedge/.style={
                ->,
                thick,
                dashed,
                red!75!black
            }
        ]
            \node[distribution] (draw) at (0,0)
            {
                fresh state
                \\[0.5mm]
                $x\sim\rho_B$
            };
            \node[state] (x) at (4.2,0) {$x$};
            \node[success] (correct) at (8.3,1.35)
            {
                take $a^\star(x)$
                \\[0.5mm]
                receive reward $u_B$
            };
            \node[terminal] (end) at (8.3,-1.35)
            {$x_{\mathrm{end}}$};
            \draw[->, thick]
                (draw.east) --
                node[midway, above]
                {sample}
                (x.west);
            \draw[goodedge]
                (x) to[bend left=8]
                node[midway, above, sloped]
                {$a^\star(x)$}
                (correct.west);
            \draw[goodedge]
                (correct.north west)
                .. controls (6.7,3.05) and (1.8,3.05) ..
                (draw.north east);
            \node[
                text=teal!70!black,
                fill=white,
                inner sep=2pt,
                font=\small
            ] at (4.25,3.25)
            {reset: draw $x'\sim\rho_B$};
            \draw[badedge]
                (x) to[bend right=9]
                node[midway, below, sloped]
                {$a\neq a^\star(x)$}
                (end.west);
            \node[
                anchor=west,
                align=left,
                font=\scriptsize,
                text=red!75!black
            ] at (9.0,-1.35)
            {
                absorbing state\\
                zero reward thereafter
            };
            \node[
                rounded corners=3pt,
                draw=black!35,
                fill=gray!3,
                font=\scriptsize,
                align=center
            ] at (4.2,-2.45)
            {
                The same gadget is used for
                $B\in\{\mathrm c,1,2\}$.
            };
        \end{tikzpicture}

        \caption{
            The hard construction inside each block from \cite{rajaraman2020toward}.
        }
        \label{fig:lower_bound_reset}
    \end{subfigure}

    \caption{
    \textbf{Recurrent hard-state construction for the matching lower bound.}
    \textbf{(a)} Expert $\pi_\ell$ enters the common block with probability
    $p_\ell$ and its private block with probability $1-p_\ell$.
    The common block is informed by all $N$ samples, whereas the private
    block of expert $\ell$ is informed only by its own $N_\ell$ samples.
    \textbf{(b)} The hard-state gadget is adapted from the standard
    imitation-learning lower-bound construction of
    \cite{rajaraman2020toward}. Inside each block, taking the correct
    action yields reward $u_B$ and resets the process to a fresh state
    drawn from the same block distribution $\rho_B$, while an incorrect
    action transitions to the absorbing zero-reward state
    $x_{\mathrm{end}}$. The recurrent reset is the modification that
    produces the matching $(1-\gamma)^{-2}$ dependence.
    }
    \label{fig:matching_lower_bound_construction}
\end{figure*}
\clearpage

\section{MA-BC with Stochastic Experts}
\label{app:stochastic_experts}

In the main text, we assumed deterministic experts for ease of exposition. Here, we formally discuss how MA-BC handles stochastic experts.

Because the return space $\mathbb{J}$ of an MOMDP forms a convex polytope \cite{momdp_convexity}, any stochastic expert residing on the Pareto front can be viewed as a convex mixture of deterministic vertex policies. Crucially, a single stochastic expert on the Pareto front maximizes a specific scalarized reward function defined by some weight vector $w \in \mathbb{R}^d_+$. If this expert selects among multiple actions at a given state, it implies that all such actions are optimal with respect to $w$.

When the expert identities are known (the split data setting of Algorithm \ref{alg:main}), intra-expert stochasticity poses no issue. Because all actions played by a single stochastic expert at a given state are optimal under the same scalarization $w$, the Markov property ensures that any policy independently selecting among these actions at each state remains optimal for $w$ (as leveraged in the proof of Theorem \ref{thm:pareto_path}). Consequently, mixing or averaging these actions at the state level does not cause the policy's expected return to fall off the Pareto front. Standard MA-BC (\Cref{alg:main}) can safely pool these shared stochastic actions alongside other experts' non-conflicting data, perfectly recovering a policy on the same optimal face of the Pareto front.

We note that the stochastic expert setting can also be analyzed directly. Using standard behavioral cloning analyses based on the logarithmic loss \cite{foster2024behavior}, one could establish performance bounds for MA-BC (similar to \Cref{thm:multi_BC}) that accommodate stochastic expert policies by bounding the statistical discrepancy (e.g., via Kullback-Leibler divergence or Total Variation distance) between the learned and expert action distributions. However, we opt to maintain the deterministic expert assumption in our main theoretical presentation for simplicity.

\newpage

\section{MA-BC with Unsplit Expert Data}
\label{app:unsplit_mabc}

In the standard setting of Multi-Output Augmented Behavioral Cloning (MA-BC) presented in Section \ref{sec:algo}, we assume the dataset $\mathcal{D}$ is already partitioned by expert identity (i.e., we know which trajectories belong to which expert). We now generalize our analysis to the fully unsplit data setting. We consider a mixed dataset generated by $L$ latent deterministic expert policies that differ at a set of distinct states, denoted by the divergence set $\mathcal{X}_{\mathrm{div}} = \mathcal{X} \setminus \mathcal{X}_{\mathrm{common}}$. We assume the expert identities are hidden and $\mathcal{X}_{\mathrm{div}}$ is initially unknown. Because MA-BC relies fundamentally on state-level data partitioning, we can approach this problem by directly identifying where the experts disagree. We identify the divergent states $\mathcal{X}_{\mathrm{div}}$ by scanning the mixed dataset for states where multiple distinct actions are observed. Once $\mathcal{X}_{\mathrm{div}}$ is identified, we can safely isolate the common data $\mathcal{D}_{\mathrm{common}}$ where all experts agree.

However, within the diverging states, we face the \textit{Binding Problem}: we know \textit{where} the experts disagree, but we do not know \textit{which} expert takes \textit{which} set of actions. To resolve this, we must construct a consistent set of divergent actions for each expert. We achieve this by leveraging the transitivity of trajectories. If a single trajectory visits $x_1$ and takes action $a_1$, and later visits $x_2$ taking action $a_2$, we can bind the state-action pairs $(x_1, a_1)$ and $(x_2, a_2)$ to the same latent expert.

We formalize this using a \textbf{Consistency Graph}, visualized in Figure \ref{fig:action_consistency}. Each unique divergent state-action pair $(x, a)$ observed in $\mathcal{X}_{\mathrm{div}}$ is a node. An edge connects $(x_i, a_i)$ and $(x_j, a_j)$ if there exists at least one trajectory $\tau \in \mathcal{D}$ that visits both states and executes those specific actions. The connected components of this graph naturally cluster the divergent state-actions into consistent, expert-specific sets $\mathcal{D}_{\mathrm{div}, \ell}$. This state-level partitioning algorithm is formalized in Algorithm \ref{alg:unsplit_mabc}.

\paragraph{Assumption (Pairwise Distinct Actions)} To guarantee that the Consistency Graph perfectly isolates $L$ latent experts when $L \ge 3$, we assume that all experts take mutually exclusive actions at the divergent states. Formally, for any divergent state $x \in \mathcal{X}_{\mathrm{div}}$ and any two distinct experts $\ell, \ell' \in [L]$, it holds that $\pi_\ell(x) \neq \pi_{\ell'}(x)$. This structural assumption prevents a scenario where two distinct experts share a single node $(x, a_{shared})$ in the graph, which would otherwise merge their components and reintroduce \textbf{\textcolor{red!70!black}{Failure II}}. For the two-expert case ($L=2$), this assumption is satisfied by definition, as a state is only flagged as divergent if the two experts disagree. We highlight that this assumption can be removed at the cost of building a consistency graph where each node is the sequence of visited divergent states in a single trajectory. Two nodes would then be connected if they share at least one divergent state and they do not contradict at any other divergent state. However, instead of graph partitioning, one would need to extract Maximal Cliques  \cite{cliques} in the graph, which is an NP-hard problem. Thus, we restrict our setting to pair-wise distinct actions, and leave the extension to arbitrary experts for future work.

\begin{figure}[ht]
    \centering
    \begin{tikzpicture}[
        scale=1.0, transform shape,
        sa_node/.style={rectangle, rounded corners, draw=black!80, thick, minimum width=1.5cm, minimum height=0.8cm, font=\bfseries, inner sep=4pt, fill=white},
        c1/.style={sa_node, fill=cyan!10, draw=cyan!80!black},
        c2/.style={sa_node, fill=red!10, draw=red!80!black},
        link/.style={-, thick, black!60},
        groupbox/.style={rounded corners, thick, inner sep=0.4cm, fill opacity=0.3}
    ]
        \node[c1] (n1) at (0, 1.5) {$(x_1, a_1)$};
        \node[c1] (n2) at (1.5, 3) {$(x_2, a_2)$};
        \node[c1] (n3) at (1.5, 0) {$(x_3, a_3)$};
        \node[c1] (n4) at (3.0, 1.5) {$(x_4, a_4)$};
        \draw[link] (n1) -- (n2) node[midway, above left, font=\scriptsize, text=black!50] {$\tau_A$};
        \draw[link] (n1) -- (n3);
        \draw[link] (n2) -- (n4);
        \draw[link] (n3) -- (n4) node[midway, below right, font=\scriptsize, text=black!50] {$\tau_B$};
        \draw[link] (n2) -- (n3);
        \node[c2] (n5) at (8, 1.5) {$(x_1, a'_1)$};
        \node[c2] (n6) at (9.5, 3) {$(x_2, a'_2)$};
        \node[c2] (n7) at (9.5, 0) {$(x_3, a'_3)$};
        \node[c2] (n8) at (11, 1.5) {$(x_4, a'_4)$};
        \draw[link] (n5) -- (n6);
        \draw[link] (n5) -- (n7) node[midway, below left, font=\scriptsize, text=black!50] {$\tau_C$};
        \draw[link] (n6) -- (n8);
        \draw[link] (n7) -- (n8);
        \draw[link] (n6) -- (n7);
        \draw[dashed, ultra thick, black!60] (5.5, -0.5) -- (5.5, 3.5);
        \node[align=center, font=\small, fill=white, inner sep=3pt] at (5.5, 1.5) {No Linking Trajectories\\(Mutually Exclusive)};
        \begin{scope}[on background layer]
            \node[text=cyan!80!black, font=\bfseries] at (1.5, 3.8) {Cluster 1 (Expert 1 Actions)};
            \node[text=red!80!black, font=\bfseries] at (9.5, 3.8) {Cluster 2 (Expert 2 Actions)};
        \end{scope}
    \end{tikzpicture}
    \caption{\textbf{Action Consistency Graph on $\mathcal{X}_{\mathrm{div}}$.} The nodes represent the conflicting state-action pairs observed at the divergent states. Edges connect actions that were observed co-occurring within the same trajectory. The graph partitions the conflicting actions into consistent, expert-specific datasets ($\mathcal{D}_{\mathrm{div}, \ell}$), successfully binding the latent experts' actions on the divergent set of states.}
    \label{fig:action_consistency}
\end{figure}
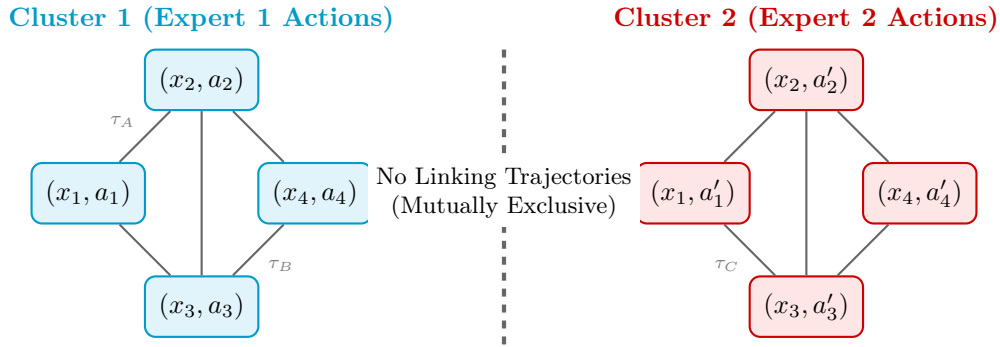

\begin{algorithm}[H]
\caption{State-Level MA-BC for Unsplit Expert Data}
\label{alg:unsplit_mabc}
\begin{algorithmic}[1]
\Require Mixed Dataset $\mathcal{D} = \{ \tau_i \}_{i=1}^N$
\State \textbf{1. Identify Diverging \& Common States:}
\State $\mathcal{X}_{\mathrm{div}} \gets \{ x \in \mathcal{X} \mid \exists (x,a), (x,a') \in \mathcal{D} \text{ such that } a \neq a' \}$
\State $\mathcal{D}_{\mathrm{common}} \gets \{ (x, a) \in \mathcal{D} \mid x \notin \mathcal{X}_{\mathrm{div}} \}$
\State $\mathcal{D}_{\mathrm{div}} \gets \{ (x, a) \in \mathcal{D} \mid x \in \mathcal{X}_{\mathrm{div}} \}$

\State \textbf{2. Build Action Consistency Graph:}
\State Construct graph $G$ where nodes are the unique state-action pairs $v = (x, a) \in \mathcal{D}_{\mathrm{div}}$.
\State Add an edge $(v_i, v_j)$ if $\exists \tau \in \mathcal{D}$ that visits both $v_i$ and $v_j$.

\State \textbf{3. Partition Divergent Actions:}
\State Find the connected components $C_1, \dots, C_L$ of $G$.
\State Let $\mathcal{D}_{\mathrm{div}, \ell} \gets \{ (x, a) \in \mathcal{D}_{\mathrm{div}} \mid (x,a) \in C_\ell \}$

\State \textbf{4. Train Experts (MA-BC):}
\For{$\ell = 1, \dots, L$}
    \State $\hat{\pi}_\ell = \argmax_{\pi \in \Pi} \sum_{X,A \in \mathcal{D}_{\mathrm{common}} \cup \mathcal{D}_{\mathrm{div}, \ell}} \log \pi(A|X)$
\EndFor
\State \textbf{Output:} $\{\hat{\pi}_\ell\}_{\ell=1}^L$
\end{algorithmic}
\end{algorithm}

\subsection{Upper Bound for Unsplit MA-BC}

We now extend Theorem \ref{thm:multi_BC} to the unsplit setting. We assume that each of the $N$ independent observations is generated by first selecting latent expert $\ell$ with probability $\alpha_\ell>0$, where $\sum_{\ell=1}^L\alpha_\ell=1$, and then sampling a state-action pair from $\mu^{\pi_\ell}$. Thus, the state marginal is $\nu^{\pi_{\mathrm{data}}}=\sum_{\ell=1}^L\alpha_\ell\nu^{\pi_\ell}$, and $N_\ell:=N\alpha_\ell$ is the expected size of the latent subset $\mathcal{D}_\ell$. Furthermore, we assume the auxiliary mechanism connecting divergent states observes links as independent events during the data collection process.

For the unsplit algorithm to successfully match the statistical rate of standard MA-BC, the consistency graph must perfectly bind the divergent state-actions to their respective experts. This identifiability relies on observing "bridge" samples that link the diverging states. Let the diverging states be $\mathcal{X}_{\mathrm{div}} = \{x_1, \dots, x_K\}$. It suffices to observe the set of pairwise links $\{(x_k, x_{k+1})\}_{k=1}^{K-1}$, where a link implies an independent sampling event captures an association between $x_k$ and $x_{k+1}$.

We seek to explicitly bound the probability that the Consistency Graph fails to fully connect the divergent actions of a target expert. Let the set of diverging states be ordered as an arbitrary sequence $\mathcal{X}_{\mathrm{div}} = \{x_1, x_2, \dots, x_K\}$. To uniquely identify an expert's actions across all $K$ states, it is sufficient to observe a "chain" of connected components. Specifically, we require the presence of $K-1$ pairwise links: $(x_1, x_2), (x_2, x_3), \dots, (x_{K-1}, x_K)$. A link $(x_i, x_{i+1})$ is considered \textit{present} in the dataset $\mathcal{D}$ if there exists at least one independent sampling event $e \in \mathcal{D}$ that bridges both states. Let $E_i$ be the event that a \textit{single} sampling event from the mixed distribution bridges both $x_i$ and $x_{i+1}$. We define $p_{\mathrm{link}}$ as the worst-case probability of any such link occurring in a single independent sampling event:

$$ \mathbb{P}(E_i) \geq \min_{j \in \{1, \dots, K-1\}} \mathbb{P}_{e \sim \nu^{\pi_{\mathrm{data}}}}(x_j \text{ and } x_{j+1} \text{ linked in } e) := p_{\mathrm{link}} $$

Let $F_i$ be the event that the link $(x_i, x_{i+1})$ is completely missing from the entire dataset of $N$ independent samples. Because the $N$ samples are drawn independently and identically, the probability that all $N$ samples fail to contain this link is:

$$ \begin{aligned}
\mathbb{P}(F_i) &= \left( 1 - \mathbb{P}(E_i) \right)^N \\
&\leq (1 - p_{\mathrm{link}})^N.
\end{aligned} $$

The global identification process fails if \textbf{at least one} of the $K-1$ necessary links is missing from the dataset. Let $F_{\mathrm{global}}$ be this global failure event: $F_{\mathrm{global}} = \bigcup_{i=1}^{K-1} F_i$. By the Union Bound, we bound the global failure probability as follows:

$$ \begin{aligned}
\mathbb{P}\left( \bigcup_{i=1}^{K-1} F_i \right) &\leq \sum_{i=1}^{K-1} \mathbb{P}(F_i) \\
&\leq \sum_{i=1}^{K-1} (1 - p_{\mathrm{link}})^N \\
&= (K-1)(1 - p_{\mathrm{link}})^N.
\end{aligned} $$

Thus, the probability that the consistency graph is "broken" (failing to fully connect the expert's divergent actions) is bounded from above by:

$$ \mathbb{P}[\text{Consistency Graph Broken}] \leq (K-1)(1 - p_{\mathrm{link}})^N $$

We will leverage this failure probability in the following theorem to demonstrate that Unsplit MA-BC achieves the exact same suboptimality bound as Theorem \ref{thm:multi_BC}, offset only by an exponentially fast-decaying identification penalty, dependent on $p_{\mathrm{link}}$.

\begin{theorem}\label{thm:unsplit_mabc}
Let $\{\pi_\ell\}_{\ell=1}^L$ be Pareto front policies and let $\{\hat{\pi}_\ell\}_{\ell=1}^L$ be the output of Algorithm \ref{alg:unsplit_mabc} on an unsplit dataset of $N$ independent samples. Let $N_\ell$ be the expected number of independent state-action samples generated by expert $\ell$, and let $K = |\mathcal{X} \setminus \mathcal{X}_{\mathrm{common}}|$. Then, for all $\ell \in [L]$, we have:

$$ \begin{aligned}
\nexists \pi^\star ~~| ~~ \forall i \in [d]~~~~~J_i(\pi^\star) - \mathbb{E}\left[J_i(\hat{\pi}_\ell)\right]
&\geq \left\|\frac{\nu^{\pi_\ell}}{\nu^{\pi_{\mathrm{data}}}}\right\|_{\infty,\X_{\mathrm{common}}}\frac{\left|\mathcal{X}_{\mathrm{common}}\right| }{e (1-\gamma)^2 N} \\
&\quad + \frac{\left|\mathcal{X}\setminus\mathcal{X}_{\mathrm{common}}\right|}{e (1-\gamma)^2 N_\ell} + \frac{\nu_{\mathrm{div}}^{\mathrm{total}}(K-1)(1 - p_{\mathrm{link}})^N}{(1-\gamma)^2}
\end{aligned} $$

where $\nu_{\mathrm{div}}^{\mathrm{total}} = \sum_{x \in \mathcal{X} \setminus \mathcal{X}_{\mathrm{common}}} \nu^{\pi_\ell}(x)$ is the occupancy mass on the divergent states and $p_{\mathrm{link}}$ is defined as above.
\end{theorem}

\begin{proof}
Without loss of generality, we consider $\ell=1$. Following the exact decomposition from the proof of Theorem \ref{thm:multi_BC}, we bound the $L_\infty$ error of policy $\hat{\pi}_1$ across the state space:

$$ \begin{aligned}
(1-\gamma)^2\left\|J(\pi_1) - J(\hat{\pi}_1)\right\|_{\infty} &\leq \underbrace{\sum_{x\in \mathcal{X}_{\mathrm{common}}} \nu^{\pi_1}(x) \sum_{a\in \mathcal{A}} \pi_1(a|x) \mathds{1}\left\{a \neq \hat{\pi}_1(x)\right\}}_{\text{Term I: Shared States}} \\
&\phantom{=} + \underbrace{\sum_{x\in \mathcal{X}\setminus \mathcal{X}_{\mathrm{common}}} \nu^{\pi_1}(x) \sum_{a\in \mathcal{A}} \pi_1(a|x) \mathds{1}\left\{a \neq \hat{\pi}_1(x)\right\}}_{\text{Term II: Divergent States}}.
\end{aligned} $$

Because the algorithm seamlessly isolates $\mathcal{D}_{\mathrm{common}}$ (where no conflict exists), Term I is entirely unaffected by the latent expert identities. Applying the change of measure exactly as in Theorem \ref{thm:multi_BC}, the error on the shared states is bounded purely by the missing mass on the total dataset $N$:

$$ \mathbb{E}[\text{Term I}] \leq \left\|\frac{\nu^{\pi_1}}{\nu^{\pi_{\mathrm{data}}}}\right\|_{\infty,\X_{\mathrm{common}}}\frac{\left|\mathcal{X}_{\mathrm{common}}\right| }{e N}. $$

For Term II (the divergent states), the learner $\hat{\pi}_1$ can fail for two reasons: (1) \textit{Coverage Error}: the specific state $x$ was never visited by expert 1 in the independently sampled dataset, or (2) \textit{Identification Error}: the state was visited, but the Consistency Graph was broken, causing the state-action pair to be discarded or misassigned. We upper-bound the probability of failure by the sum of the coverage and identification failure probabilities:

$$ \begin{aligned}
\mathbb{E}[\text{Term II}] &\leq \sum_{x\in \mathcal{X}\setminus \mathcal{X}_{\mathrm{common}}} \nu^{\pi_1}(x) \left( \mathbb{P}[x \notin \mathcal{D}_1] + \mathbb{P}[\text{Consistency Graph Broken}] \right) \\
&\leq \sum_{x\in \mathcal{X}\setminus \mathcal{X}_{\mathrm{common}}} \nu^{\pi_1}(x) e^{-N_1\nu^{\pi_1}(x)} + \sum_{x\in \mathcal{X}\setminus \mathcal{X}_{\mathrm{common}}} \nu^{\pi_1}(x) (K-1)(1 - p_{\mathrm{link}})^N \\
&\leq \frac{\left|\mathcal{X}\setminus\mathcal{X}_{\mathrm{common}}\right|}{e N_{1}} + \nu_{\mathrm{div}}^{\mathrm{total}}(K-1)(1 - p_{\mathrm{link}})^N,
\end{aligned} $$

The first inequality is the union bound. For the second, we use the graph bound and
\[
\mathbb{P}[x\notin\mathcal{D}_1]
=(1-\alpha_1\nu^{\pi_1}(x))^N
\leq e^{-N_1\nu^{\pi_1}(x)}.
\]
The final step uses $v e^{-N_1v}\leq 1/(eN_1)$ for $v\geq0$.  Summing the terms and dividing by $(1-\gamma)^2$ establishes the expected $L_\infty$ bound on the return. This confirms that Unsplit MA-BC achieves the optimal statistical rate $\mathcal{O}(1/N_\ell + 1/N)$ up to an exponentially vanishing identification penalty.
\end{proof}

It remains open whether the dependence on $p_{\mathrm{link}}$ is necessary, and we highlight that providing a lower bound in the latent experts setting remains as an exciting future direction.

\newpage
\section{Experimental Details}
\label{app:environments}

\subsection{Environment Details}

In this section, we provide a detailed description of the state spaces, action spaces, transition dynamics, and multi-objective reward structures for all the domains used in our empirical evaluation in Section \ref{sec:experiments}. All experiments were run on a desktop CPU with a negligible runtime.

\begin{figure}[h!]
    \centering
        \begin{subfigure}[b]{0.31\textwidth}
        \centering
        \includegraphics[width=\linewidth]{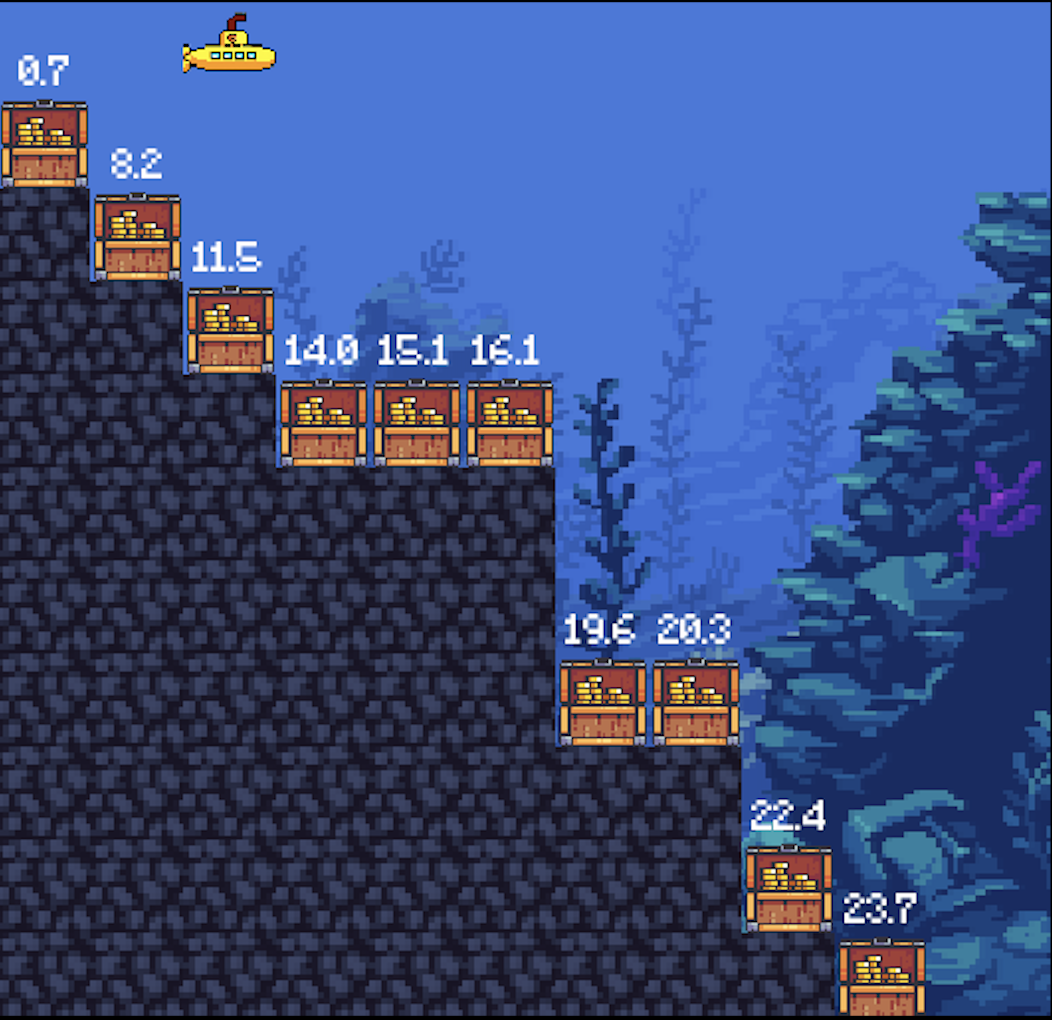}
        \caption{Deep Sea Treasure}
        \label{fig:env_deep_sea}
    \end{subfigure}
    \hfill
    \begin{subfigure}[b]{0.315\textwidth}
        \centering
        \includegraphics[width=\linewidth]{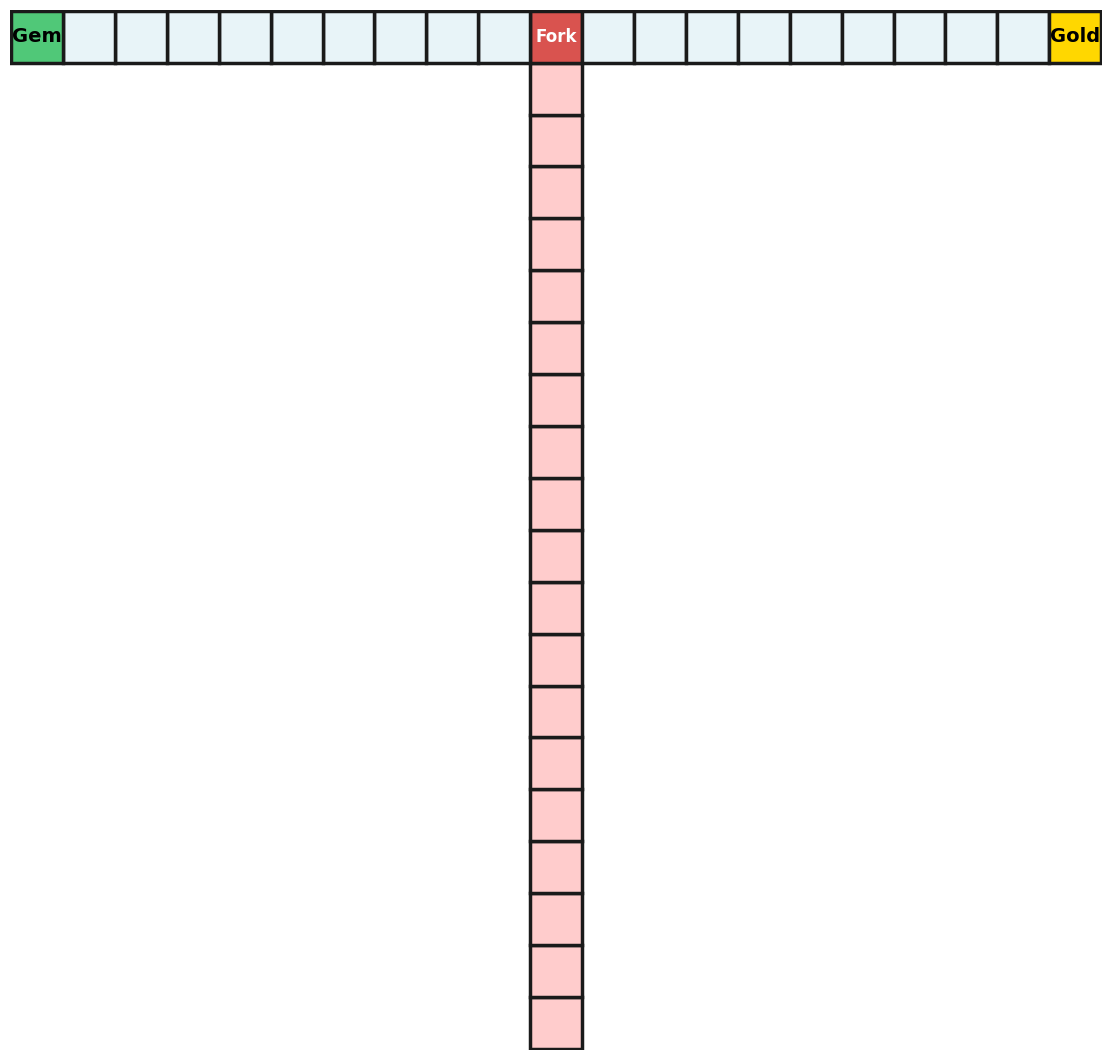}
        \caption{Slippery Y-Maze}
        \label{fig:env_y_maze}
    \end{subfigure}
    \hfill
    \begin{subfigure}[b]{0.3\textwidth}
        \centering
        \includegraphics[width=\linewidth]{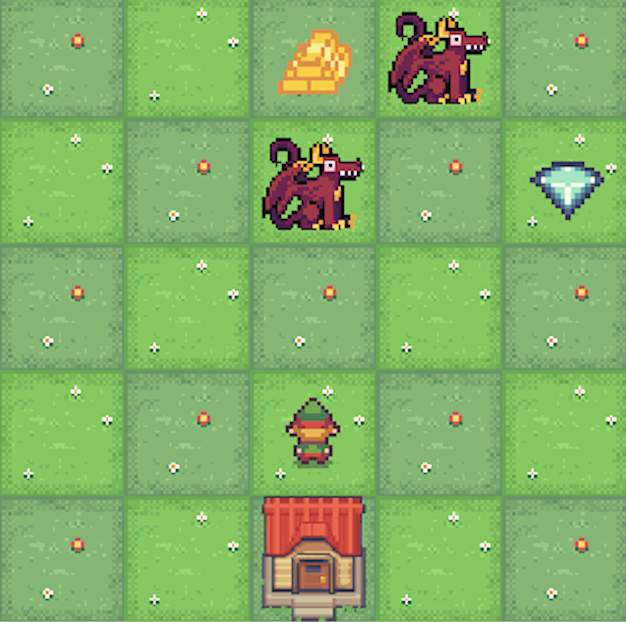}
        \caption{Resource Gathering}
        \label{fig:env_resource}
    \end{subfigure}

    \caption{Overview of the discrete multi-objective environments used for empirical evaluation. \textbf{(a)} Deep Sea Treasure requires balancing the treasure value against a time penalty. \textbf{(b)} The Slippery Y-Maze introduces a stochastic corridor where the agent must navigate, and choose between the gem and the gold. \textbf{(c)} In the Resource Gathering environment, the agent receives a 3-D reward signal regarding collecting the gold, the gem and avoiding the enemies.}
    \label{fig:discrete_environments}
\end{figure}

\subsubsection{Deep Sea Treasure}
Deep Sea Treasure \cite{deep_sea_env, vamplew2011empirical} is a standard, widely used benchmark in multi-objective reinforcement learning. The environment consists of a discrete grid-world where a submarine agent navigates to find treasures located on the sea floor.
\begin{itemize}
    \item \textbf{State Space:} The discrete $(x, y)$ coordinates of the submarine.
    \item \textbf{Action Space:} Four discrete directional actions: Up, Down, Left, and Right.
    \item \textbf{Reward Vector:} A 2-dimensional vector $r(x, a) = [r_{\text{treasure}}, r_{\text{time}}]^T$. The agent receives a time penalty of $r_{\text{time}} = -1$ for every step taken. The treasure reward $r_{\text{treasure}}$ is strictly positive only when the agent reaches a terminal treasure state. Treasures located further from the start yield exponentially higher values, creating a strictly convex Pareto front that explicitly trades off time efficiency against total treasure value.
\end{itemize}

\subsubsection{Slippery Y-Maze}
\label{app:y_maze_details}

The Slippery Y-Maze is a discrete grid-world environment designed to introduce severe stochastic bottlenecks and spatial aliasing, testing the robustness of multi-objective imitation learning algorithms under risk.

\paragraph{State Space}
The state space $\mathcal{X}$ consists of a 2D coordinate grid shaped as a 'Y', plus a dedicated absorbing terminal state $x_{\text{term}}$. The grid is divided into three distinct regions: the stem, the left branch (Gem path), and the right branch (Gold path). We use the parameters ($L_{\text{stem}} = 20$, $L_{\text{branch}} = 10$) in our experiments. The agent spawns according to a uniform initial state distribution over all non-terminal states.

\paragraph{Action Space}
The agent has four discrete directional actions $\mathcal{A} = \{\text{Up}, \text{Down}, \text{Left}, \text{Right}\}$. We denote the deterministic coordinate shift associated with an action $a$ as $\Delta_a$ (e.g., $\Delta_{\text{Up}} = (0, 1)$).

\paragraph{Transition Dynamics}
The environment introduces a stochastic bottleneck in the stem to punish dithering. While the agent is in the stem ($S_{\text{stem}}$), every action carries a 10\% ambient risk of ``falling off'' or dying, which transitions the agent immediately to $x_{\text{term}}$ with zero reward. Once the agent reaches the branches, the ambient risk drops to zero and transitions become deterministic.

Formally, the transition probabilities $P(x' \mid x, a)$ are defined as follows. Let $x_{\text{next}} = x + \Delta_a$. If $x_{\text{next}} \notin \mathcal{X}$, the agent collides with a wall and remains in its current state (i.e., $x_{\text{next}} = x$).
\begin{itemize}
    \item \textbf{In the Stem ($x \in S_{\text{stem}}$):}
    \begin{align*}
        P(x_{\text{term}} \mid x, a) &= 0.10 \\
        P(x_{\text{next}} \mid x, a) &= 0.90
    \end{align*}
    \item \textbf{In the Branches ($x \in S_{\text{left}} \cup S_{\text{right}}$):}
    \begin{align*}
        P(x_{\text{next}} \mid x, a) &= 1.0
    \end{align*}
\end{itemize}
When the agent steps off the ends of the branches (coordinates $(10, 19)$ or $(-10, 19)$), it transitions deterministically to $x_{\text{term}}$ and the episode ends.

\paragraph{Reward Vector}
The reward function $r(x, a, x') \in \mathbb{R}^2$ represents a trade-off between two objectives: [Gold, Gem]. The environment yields zero reward everywhere except on the transitions exiting the extremities of the two branches:
\begin{align*}
    r((10, 19), a, x_{\text{term}}) &= [1.0, 0.0]^T \quad \text{(Gold Expert Target)} \\
    r((-10, 19), a, x_{\text{term}}) &= [0.0, 1.0]^T \quad \text{(Gem Expert Target)}
\end{align*}
All other transitions yield $[0.0, 0.0]^T$. The discount factor is set to $\gamma = 0.9999$ to simulate an effectively undiscounted shortest-path problem.

\paragraph{Expert Policies}
By finding the exact continuous Pareto frontier using Optimistic Linear Support (OLS), we extract two deterministic target experts at the extremes of the frontier:
\begin{enumerate}
    \item \textbf{Gold Expert ($\pi_A^*$):} Always routes right at the central fork $(0, 19)$ to secure the $[1.0, 0.0]^T$ reward.
    \item \textbf{Gem Expert ($\pi_B^*$):} Always routes left at the central fork $(0, 19)$ to secure the $[0.0, 1.0]^T$ reward.
\end{enumerate}
Crucially, because both experts want to minimize time spent in the dangerous stem region (to avoid the 10\% ambient death probability), they share the exact same optimal policy for all states in $S_{\text{stem}}$ (always moving Up). The only point of divergence occurs at the fork $(0,19)$. If an imitation learner naively pools these experts (Naive BC), the conflicting actions at the fork will cause the agent to dither, repeatedly triggering the 10\% death penalty and falling entirely off the Pareto frontier.
\subsubsection{Resource Gathering}
Resource Gathering tests the scalability of MA-BC by extending the conflict into a 3-dimensional objective space, moving beyond simple 2D trade-offs.
\begin{itemize}
    \item \textbf{State Space:} A discrete grid containing multiple spawn points for two distinct types of resources (Gold and Gems) as well as roaming enemies.
    \item \textbf{Action Space:} Four discrete directional actions.
    \item \textbf{Reward Vector:} A 3-dimensional vector $r(x, a) = [r_{\text{gold}}, r_{\text{gem}}, r_{\text{enemy}}]^T$. The agent receives $+1$ on the respective channel for collecting a Gold or Gem, and a severe $-1$ penalty on the third channel if it encounters an enemy. The Pareto-optimal experts in our dataset demonstrate distinct, specialized behaviors: one strictly prioritizes gathering Gold, one exclusively hunts Gems, and one exhibits risk-averse behavior by prioritizing enemy avoidance over resource collection.
\end{itemize}

\subsubsection{LQR Drone Control}
\label{app:drone_lqr_details}

In our continuous-control experiments, we model a linearized 6 Degrees of Freedom (6-DOF) quadcopter. The continuous-time linear time-invariant (LTI) system is defined as $\dot{x} = A_c x + B_c u$.

\paragraph{State and Action Space}
The state vector $x \in \mathbb{R}^{12}$ consists of the Euler angles (roll, pitch, yaw), the 3D position, the body angular rates, and the linear velocities:
\[
x = [\phi, \theta, \psi, X, Y, Z, p, q, r, u, v, w]^T
\]
The action vector $u \in \mathbb{R}^4$ consists of the vertical thrust and the torques applied around the three body axes:
\[
u = [F, \tau_x, \tau_y, \tau_z]^T
\]

\paragraph{Continuous Dynamics}
The continuous dynamics matrices $A_c \in \mathbb{R}^{12 \times 12}$ and $B_c \in \mathbb{R}^{12 \times 4}$ encode the kinematic mappings and gravity couplings (where $g = 9.8$ m/s$^2$). The non-zero derivatives of the state variables are defined as:
\begin{align*}
    \dot{\phi} &= p, \quad \dot{\theta} = q, \quad \dot{\psi} = r \\
    \dot{X} &= u, \quad \dot{Y} = v, \quad \dot{Z} = w \\
    \dot{u} &= -g \theta, \quad \dot{v} = g \phi \\
    \dot{p} &= 10 \tau_x, \quad \dot{q} = 10 \tau_y, \quad \dot{r} = 6.6667 \tau_z, \quad \dot{w} = 5 F
\end{align*}
To formulate the discrete-time MDP $x_{t+1} = A x_t + B u_t$, we discretize the continuous dynamics using Euler integration with a time step of $\Delta t = 0.05$ seconds:
\[
A = I + A_c \Delta t, \quad B = B_c \Delta t
\]

\paragraph{LQR Objective Matrices}
The reward function is defined by the standard LQR quadratic penalty $r(x, u) = -(x^T Q x + u^T R u)$. We define a standard, uniform control effort penalty $R = I_{4 \times 4}$.

To create the multi-objective conflict, we define two base state-penalty matrices. The first heavily penalizes positional tracking error, while the second relaxes positional tracking to simulate an economic, energy-saving mode. We define $Q_{\text{tracking}}$ and $Q_{\text{eco}}$ as diagonal matrices:
\begin{align*}
    Q_{\text{tracking}} &= \text{diag}(10, 10, 10, 100, 100, 100, 1, 1, 1, 1, 1, 1) \\
    Q_{\text{eco}} &= \text{diag}(0.1, 0.1, 0.1, 1, 1, 1, 0.1, 0.1, 0.1, 0.1, 0.1, 0.1)
\end{align*}

\paragraph{Expert Generation}
To test the agents under a controlled degree of objective conflict, we introduce a divergence parameter $\alpha \in [0, 1]$. When $\alpha=0$, the experts are identical; when $\alpha=1$, they are maximally divergent. In our experiments, we use $\alpha = 0.5$. The specific reward matrices for the two experts are generated via convex combinations:
\begin{align*}
    Q_1 \text{ (Agile Expert)} &= \left(0.5 + 0.5\alpha\right) Q_{\text{tracking}} + \left(0.5 - 0.5\alpha\right) Q_{\text{eco}} \\
    Q_2 \text{ (Eco Expert)} &= \left(0.5 - 0.5\alpha\right) Q_{\text{tracking}} + \left(0.5 + 0.5\alpha\right) Q_{\text{eco}}
\end{align*}
Given $A, B, R$, and the respective $Q_i$, the true expert controllers $K_1$ and $K_2$ are obtained by solving the Discrete Algebraic Riccati Equation (DARE).

\subsection{Exact $L_\infty$ Distance to the Pareto Front via Linear Programming}
\label{app:linf_lp}

In our empirical evaluations, we measure the suboptimality of a learned policy $\pi$ by computing its normalized $L_\infty$ distance to the continuous Pareto frontier. Because the exact continuous frontier is defined by the convex hull of the discrete, deterministic Pareto-optimal vertices, we cannot simply measure the distance to the nearest discrete vertex. Doing so would unfairly penalize policies that correctly learn an optimal stochastic mixture of experts along a continuous edge.

To compute the exact distance to the continuous frontier, we formulate the problem as a Linear Program (LP). Let $\mathcal{V} = \{v_1, v_2, \dots, v_{M}\}$ be the set of expected return vectors for the $M$ deterministic Pareto-optimal policies (we write $M$ to avoid a clash with the dataset size $N$). Let $J(\pi) \in \mathbb{R}^d$ be the expected return vector of the evaluated policy.

To adjust for the different scales of the objectives, we compute the objective ranges $R_i = J_i^{\text{max}} - J_i^{\text{min}}$, where $J^{\text{max}}$ and $J^{\text{min}}$ are the coordinate-wise maximum and minimum of the Pareto vertices, respectively. The vertices on the Pareto front can be computed using the Optimistic Linear Support algorithm \cite{roijersPhD}.

We seek to find the maximum normalized gap $\delta \ge 0$ such that there exists a valid point on the continuous Pareto frontier (a convex combination of the vertices $\mathcal{V}$) that dominates the learned policy $J(\pi)$ by at least $\delta$ across all $d$ normalized objectives.

This is computed via the following Linear Program:
\begin{align*}
    \max_{\delta, \alpha} \quad & \delta \\
    \text{subject to} \quad & \sum_{j=1}^{M} \alpha_j v_{j, i} \ge J(\pi)_i + \delta \cdot R_i, \quad \forall i \in [d] \\
    & \sum_{j=1}^{M} \alpha_j = 1 \\
    & \alpha_j \ge 0, \quad \forall j \in \{1, \dots, M\}
\end{align*}
where $\alpha \in \mathbb{R}^{M}$ represents the mixture weights over the discrete Pareto vertices. By maximizing $\delta$, we find the exact worst-case normalized suboptimality gap between $J(\pi)$ and the true continuous boundary of the return polytope $\mathbb{J}$.

\subsection{A Pedagogical Deep Dive into the Deep Sea Treasure Environment}
\label{app:pedagogical_deep_dive}

In this section, we present a detailed analysis of the Deep Sea Treasure (DST) environment (Figure \ref{fig:env_deep_sea}), consolidating our empirical and theoretical findings into a cohesive, pedagogical narrative.

The DST agent receives a 2D reward signal: the value of the collected treasure and a constant step penalty for the time spent navigating. To establish the ground truth, we compute the exact Pareto vertices using Optimistic Linear Support (OLS) \cite{roijersPhD}. Because we have access to the environment's transition dynamics, we evaluate the scalarized objectives precisely via Value Iteration \cite{puterman_mdp}. Assuming a uniform initial state distribution over all non-terminal states and a discount factor of $\gamma = 0.999$, we reconstruct the complete Pareto front, visualized in \Cref{fig:dst_pf}.

\begin{figure}[h]
    \centering
    \includegraphics[width=\linewidth]{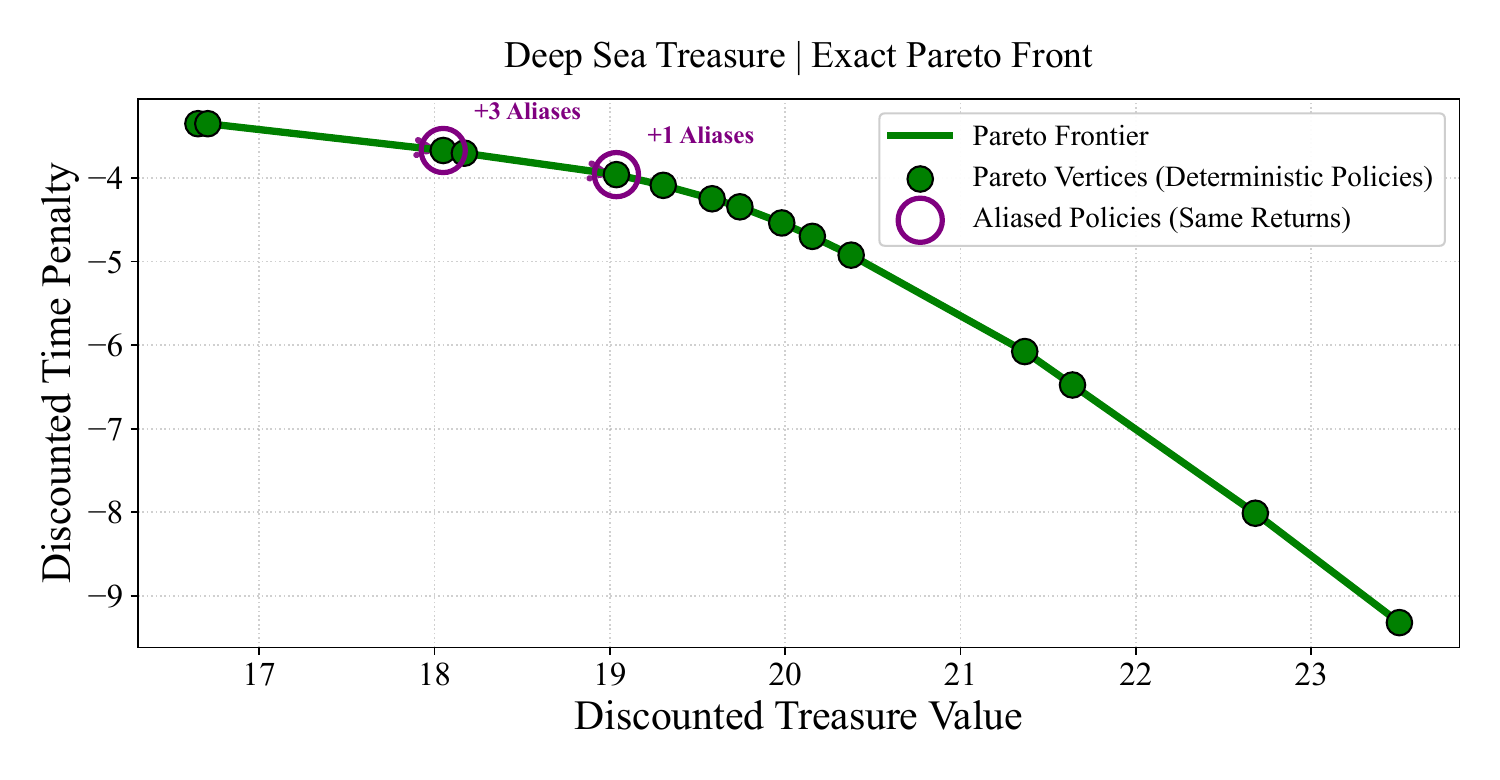}
    \caption{\textbf{The DST Return Polytope.} The exact continuous Pareto frontier is bounded by deterministic optimal policies (green vertices). Aliased policies (where multiple distinct optimal policies yield the exact same expected return) are highlighted with purple circles.}
    \label{fig:dst_pf}
\end{figure}

\paragraph{Pareto front structure} First, we empirically validate the structural properties introduced in \Cref{sec:preliminaries}, namely, that the Pareto front is densely populated by structurally similar policies. We observe a compelling geometric phenomenon: almost all Pareto-optimal policies differ from their immediate neighbors at exactly one state. The only exceptions are two specific vertices (highlighted with purple circles in \Cref{fig:dst_pf}). This observation demonstrates that the structural guarantee in \cite[Theorem 1]{finding_pareto_front} does not strictly hold in highly structured environments like DST, where neighboring policies on the Pareto front may differ by more than one state. This is not a flaw in their proof, but rather a consequence of their measure-theoretic assumption that each return vector is uniquely attained by a single policy (thus confining structured environments to a measure-zero set of MOMDPs). Upon further investigation, we found that the vertices violating the 1-flip rule are heavily \emph{aliased}, meaning multiple distinct deterministic policies achieve the exact same expected return vector. This insight directly motivated our generalized ``Pareto Path'' guarantee (\Cref{thm:pareto_path}), which ensures 1-flip connectivity even in the presence of severe aliasing. Notably, policy-based search methods like Pareto Traversal \cite[Algorithm 1]{finding_pareto_front} become computationally prohibitive here, as tracking all aliased combinations induces an exponential branching factor. This justifies our use of a value-based scalarization algorithm (OLS), which remains entirely agnostic to policy aliasing.

\begin{figure}[h]
    \centering
    \includegraphics[width=\linewidth]{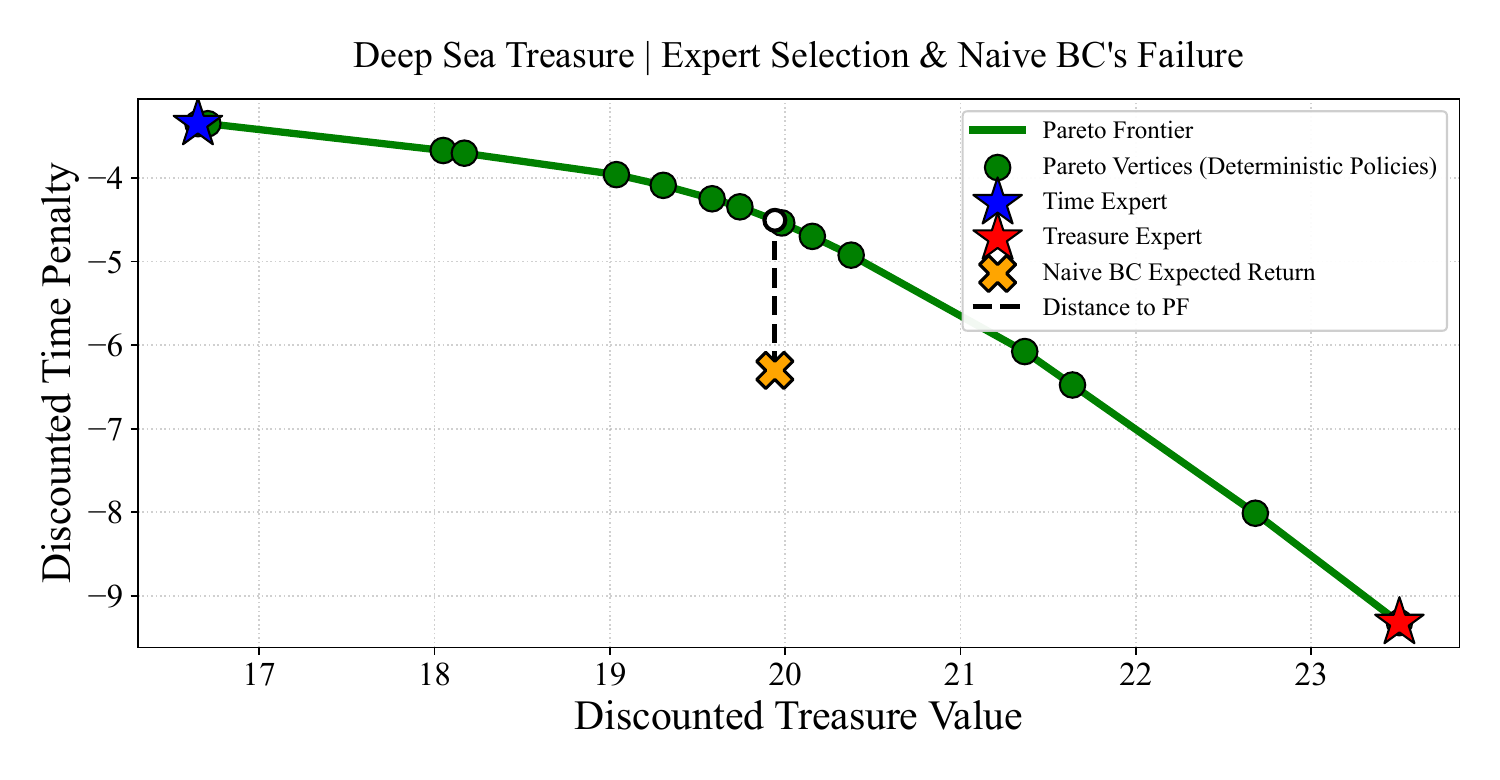}
    \caption{\textbf{The Geometric Failure of Naive BC (\textcolor{red!70!black}{\textbf{Failure II}})}. We collect data from two extreme experts, highlighted using the blue and red stars. Naive BC aggregates all of this data into a single joint dataset. Because the experts fundamentally disagree at bottleneck states, the resulting Maximum Likelihood Estimate policy diffuses its probability mass destructively, causing its expected return to sag deeply into the suboptimal interior of the polytope.}
    \label{fig:dst_naive_failure}
\end{figure}

\paragraph{Imitation learning (\textcolor{red!70!black}{\textbf{Failure II}})} We now turn our attention to the imitation learning setting. We select two extreme experts: a time-optimizing expert (blue star) and a treasure-optimizing expert (red star), as visualized in \Cref{fig:dst_naive_failure}. We sample trajectories from both experts and learn a single joint policy using naive Behavioral Cloning. As predicted by \textbf{Failure II} and demonstrated in \Cref{sec:experiments}, this pooled policy completely fails to reach the Pareto front.

To understand why, we compare this dominated policy to the closest mathematically optimal policy on the continuous Pareto frontier (\Cref{fig:dst_mix_vs_optimal}). We observe a stark contrast in their stochastic behaviors. The naively learned policy is overly stochastic and uniformly randomizes its actions across all divergent states, failing to optimize any coherent utility. In contrast, the true nearest Pareto-optimal policy is also stochastic, but its stochasticity is tightly concentrated at a \emph{single} critical state. It carefully randomizes between navigating to one treasure versus the immediately adjacent one. This demonstrates that while stochastic mixing is necessary to traverse the continuous edges of the Pareto front, naive BC diffuses this entropy destructively, whereas the true optimal mixture focuses it.

\begin{figure}[h]
    \centering
    \includegraphics[width=\linewidth]{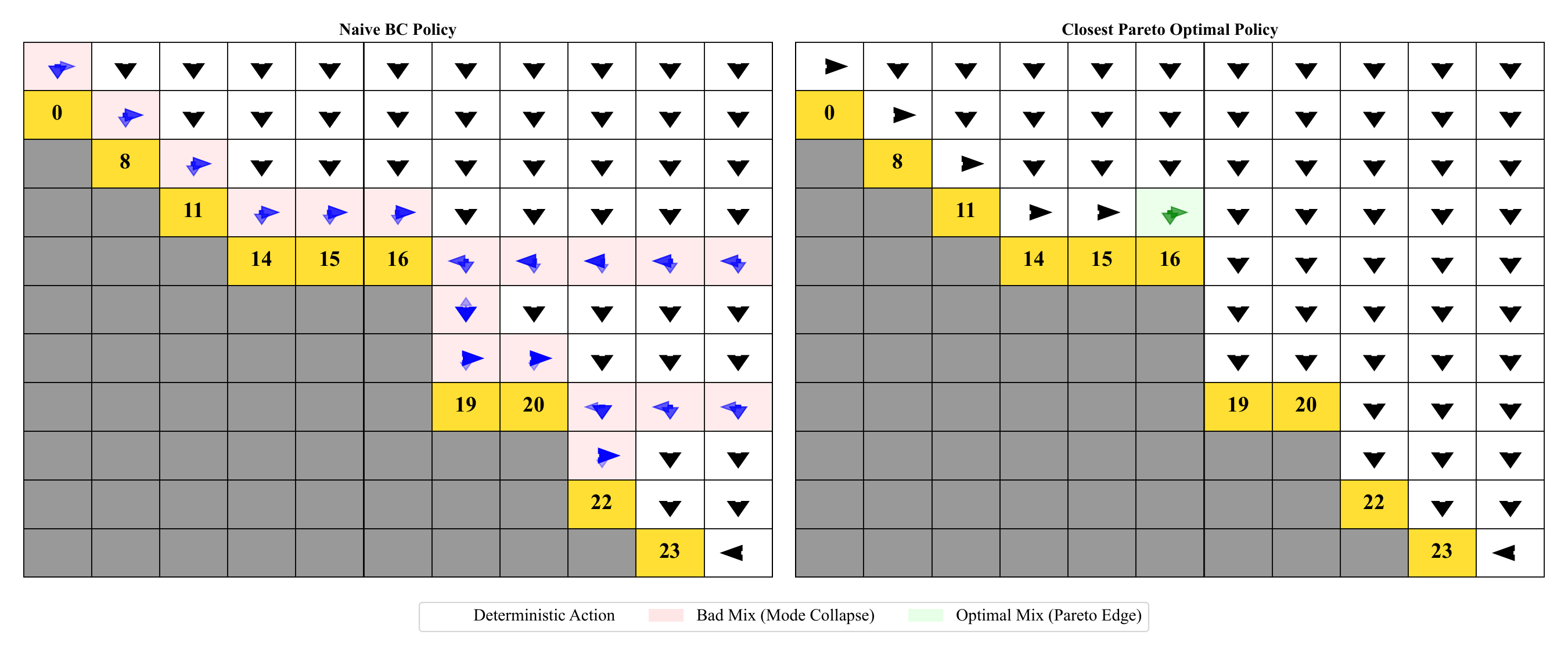}
    \caption{\textbf{\textcolor{red!70!black}{\textbf{Failure II}}: Naive BC vs. Closest Optimal Mixture.} \textbf{(Left)} The policy learned by Naive BC is overly stochastic at conflicting states (red cells). \textbf{(Right)} The closest theoretically optimal policy on the Pareto front. This policy achieves continuous returns by randomizing at only a single critical state, properly trading off between two adjacent treasures without diffusing entropy throughout the trajectory.}
    \label{fig:dst_mix_vs_optimal}
\end{figure}

\paragraph{Imitation learning (\textcolor{red!70!black}{\textbf{Failure I}})} Finally, we analyze the mechanics of MA-BC compared to Isolated BC to understand how it overcomes \textcolor{red!70!black}{\textbf{Failure I}}. We collect an extremely small dataset of five trajectories from each expert and train both algorithms. The resulting policies are visualized in \Cref{fig:dst_iso_failure_5}.

\begin{figure}[h]
    \centering
    \includegraphics[width=\linewidth]{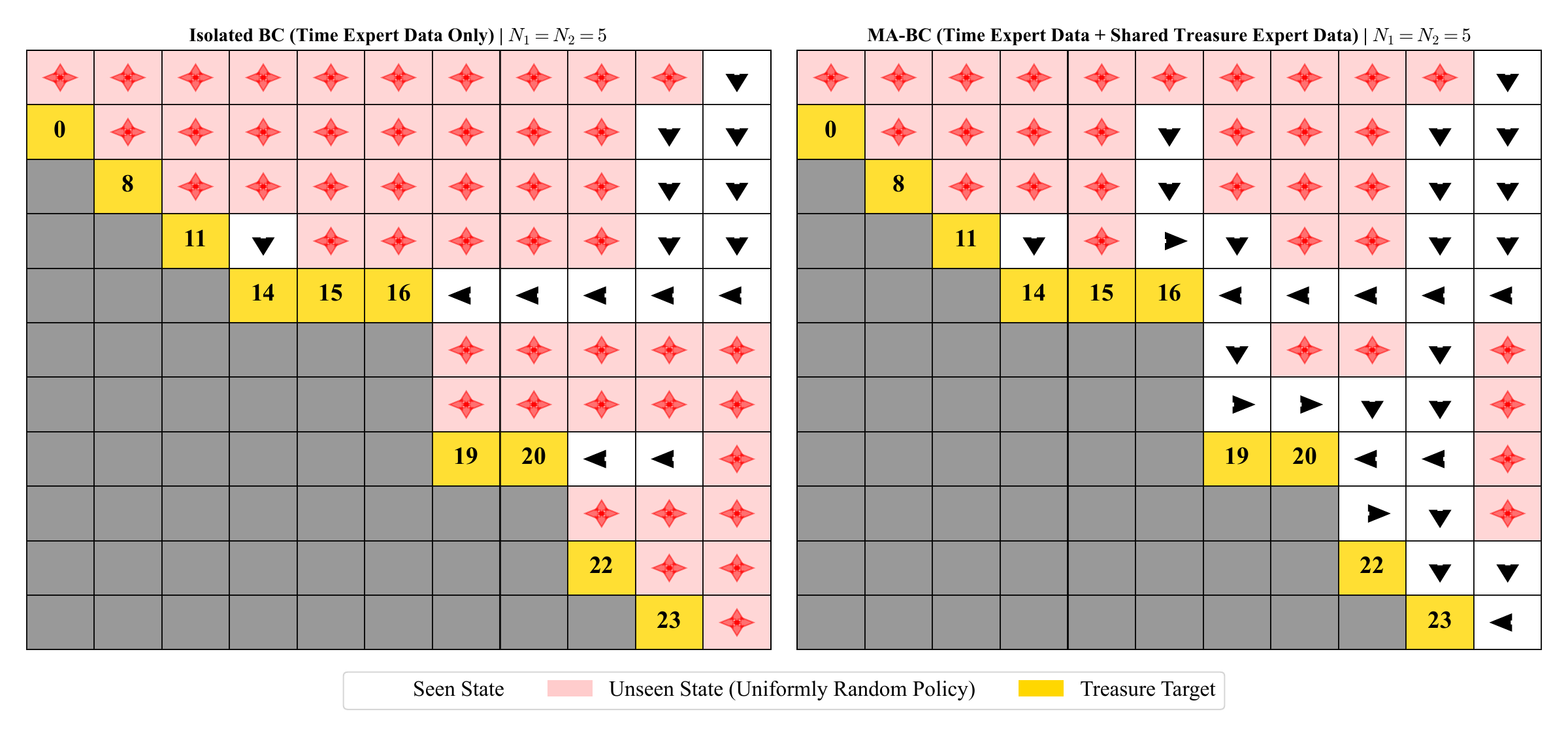}
    \caption{\textbf{\textcolor{red!70!black}{\textbf{Failure I}}: The Failure of Isolated BC ($N_1=N_2=5$)}. We depict the time-optimizing policy learned via Isolated BC (Left) versus MA-BC (Right) using only 5 trajectories. MA-BC aggressively pools non-divergent actions from the treasure-optimizing expert to fill in unvisited states. While largely helpful, MA-BC introduces temporary bias in the low-data regime (e.g., the actions above cells 16, 19, 20, and 22 incorrectly route the agent away from the nearest treasure).}
    \label{fig:dst_iso_failure_5}
\end{figure}

In the right panel of \Cref{fig:dst_iso_failure_5}, we observe that MA-BC successfully pools a massive portion of data from the treasure-optimizing expert to assist the time-minimizing learner in unvisited regions. However, this reveals a nuance: while much of the pooled data is consistent with a time-minimizing behavior, some is not. Specifically, the actions above cells (16, 19, 20, 22) direct the agent away from the nearest treasure, effectively increasing the time penalty. This illustrates that MA-BC can inadvertently inject bias into the target expert's dataset in the ultra-low data regime. The critical mechanism of MA-BC, however, is that this bias is inherently \emph{self-correcting}. When we scale the dataset to 50 trajectories per expert (\Cref{fig:dst_iso_failure_50}), the bias almost vanishes. As the time-minimizing expert explores more of the state space, it eventually generates trajectories that conflict with the previously pooled data at cells (16, 19, 20). MA-BC detects this conflict, immediately isolates those states, and reverts to the true expert's actions. This highlights the core operating principle of MA-BC: data is safely pooled until a contradiction is observed. This self-correcting property ensures that MA-BC remains a consistent estimator of the underlying expert policy asymptotically.

\begin{figure}[h]
    \centering
    \includegraphics[width=\linewidth]{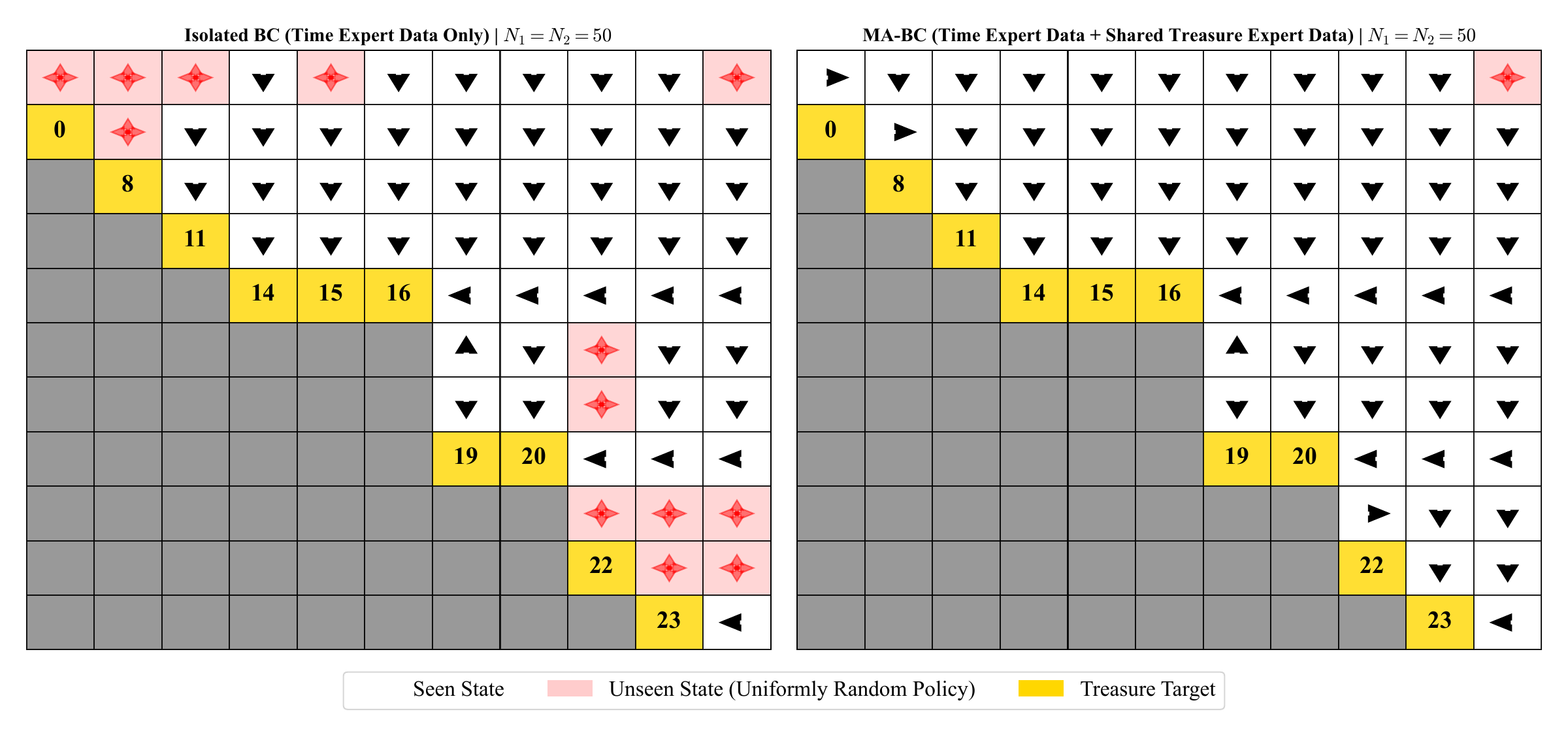}
    \caption{\textbf{MA-BC Self-Correction ($N_1=N_2=50$)}. As the sample size increases to 50 trajectories per expert, the bias previously introduced by MA-BC (\Cref{fig:dst_iso_failure_5}) is successfully self-corrected. The actions above cells 16, 19, and 20 (Right) have been overwritten to match the true time-minimizing behavior. As the target expert explores more states, conflicts are triggered, naturally partitioning the dataset and eliminating the bias.}
    \label{fig:dst_iso_failure_50}
\end{figure}

This phenomenon also provides vital context for our continuous state-action extension (\Cref{sec:continous_mabc}). In continuous spaces, the spatial tolerance threshold $\delta$ introduces a constant bias that does \emph{not} neatly self-correct, because exact state collisions almost surely have zero probability. This fundamentally explains the empirical bias-data tradeoff observed in \Cref{fig:drone_eco_summary}, dictating that as the baseline sample size increases, the continuous learner requires progressively smaller $\delta$ values to prioritize bias reduction over variance reduction.

\newpage
\end{document}